%% file: neurips_2022.tex
\documentclass{article}

    \PassOptionsToPackage{numbers, compress}{natbib}

    \usepackage[final]{neurips_2022}

\usepackage[utf8]{inputenc} 
\usepackage[T1]{fontenc}    
\usepackage{hyperref}       
\usepackage{url}            
\usepackage{booktabs}       
\usepackage{amsfonts}       
\usepackage{nicefrac}       
\usepackage{microtype}      
\usepackage{xcolor}         
\usepackage[colorinlistoftodos,textwidth=3.0cm,textsize=tiny]{todonotes}

\usepackage{macros}

\definecolor{cobalt}{rgb}{0.0, 0.28, 0.67}
\definecolor{carnelian}{rgb}{0.7, 0.11, 0.11}

\definecolor{bole}{rgb}{0.47, 0.27, 0.23}

\newcommand{\update}[1]{\textcolor{black}{#1}}

\title{Unsupervised Learning under Latent Label Shift}

\author{%
  Manley Roberts\thanks{Equal contribution.} \quad\quad
  Pranav Mani$^*$  \quad\quad
 Saurabh Garg \quad\quad
 Zachary C. Lipton \\
  Carnegie Mellon University \\
  \texttt{\{manleyroberts,zlipton\}@cmu.edu};\ \texttt{\{pmani, sgarg2\}@cs.cmu.edu}
}

\begin{document}

\maketitle

\begin{abstract}
\input{sections/00_abstract}
\end{abstract}

\section{Introduction}
\input{sections/01_intro}

\section{Related Work}

\input{sections/02_related}

\section{Latent Label Shift Setup}

\input{sections/03_LLS} \label{sec:setup}

\section{Theoretical Analysis}

\input{sections/04_theory} \label{sec:theory}

\section{DDFA Framework}

\input{sections/05_PracticalAlgorithm}

\section{Experiments}  \label{sec:exp}

\input{sections/06_ExperimentalSection}

\section{Conclusion}

\input{sections/07_conclusion}

\subsection{Future Work and Limitations}
\label{sec:limitations}
\input{sections/08_limitations}

\section*{Acknowledgements}

\input{sections/09_acknowledgements}

\newpage 

\bibliographystyle{plainnat}
\bibliography{references}

\newpage

\begin{enumerate}

\item For all authors...
\begin{enumerate}
  \item Do the main claims made in the abstract and introduction accurately reflect the paper's contributions and scope?
    \answerYes{}
  \item Did you describe the limitations of your work?
    \answerYes{}
  \item Did you discuss any potential negative societal impacts of your work?
    \answerNA{We believe that this work, which proposes a novel unsupervised learning problem does not present a significant societal concern. 
    While this could potentially guide practitioners
    to improve classification, 
    we do not believe that it will fundamentally 
    impact how machine learning is used in a way 
    that could conceivably be socially salient.}
  \item Have you read the ethics review guidelines and ensured that your paper conforms to them?
    \answerYes{}
\end{enumerate}

\item If you are including theoretical results...
\begin{enumerate}
  \item Did you state the full set of assumptions of all theoretical results?
    \answerYes{See \secref{sec:setup}}
        \item Did you include complete proofs of all theoretical results?
    \answerYes{See \secref{sec:theory} and appendices}
\end{enumerate}

\item If you ran experiments...
\begin{enumerate}
  \item Did you include the code, data, and instructions needed to reproduce the main experimental results (either in the supplemental material or as a URL)?
    \answerYes{We include all the necessary details to replicate our experiments in appendices. Code URL is given in the paper.}
  \item Did you specify all the training details (e.g., data splits, hyperparameters, how they were chosen)?
    \answerYes{Yes, we describe crucial details in \secref{sec:exp} and defer precise details to appendices.}
        \item Did you report error bars (e.g., with respect to the random seed after running experiments multiple times)?
    \answerYes{We include results with multiple seeds in appendices, with error bars for all main paper experiments. Error bars are not yet given for some ablations.}
        \item Did you include the total amount of compute and the type of resources used (e.g., type of GPUs, internal cluster, or cloud provider)?
    \answerYes{Refer to experimental setup in \appref{sec:appendix-experimental-details}.}
\end{enumerate}

\item If you are using existing assets (e.g., code, data, models) or curating/releasing new assets...
\begin{enumerate}
  \item If your work uses existing assets, did you cite the creators?
    \answerYes{Refer to experimental setup in \appref{sec:appendix-experimental-details}.}
  \item Did you mention the license of the assets?
    \answerYes{Refer to experimental setup in \appref{sec:appendix-experimental-details}.}
  \item Did you include any new assets either in the supplemental material or as a URL?
    \answerYes{Refer to experimental setup in \appref{sec:appendix-experimental-details}.}
  \item Did you discuss whether and how consent was obtained from people whose data you're using/curating?
    \answerNA{}
  \item Did you discuss whether the data you are using/curating contains personally identifiable information or offensive content?
    \answerNA{}
\end{enumerate}

\item If you used crowdsourcing or conducted research with human subjects...
\begin{enumerate}
  \item Did you include the full text of instructions given to participants and screenshots, if applicable?
    \answerNA{}
  \item Did you describe any potential participant risks, with links to Institutional Review Board (IRB) approvals, if applicable?
    \answerNA{}
  \item Did you include the estimated hourly wage paid to participants and the total amount spent on participant compensation?
    \answerNA{}
\end{enumerate}

\end{enumerate}

\newpage

\appendix

\section*{Supplementary Material}

\section{Proofs of Lemmas}
\label{sec:lemma-proofs}
\input{sections/appendix_A_proofs_lemmas}

\newpage

\section{Proof of Theorem \ref{thm:anchor-subdomain-continuous}}
\label{sec:theorem-2-proof_app}
\input{sections/appendix_B_proof_theorem}

\section{Minimizing Cross-Entropy Loss yields Domain Discriminator}
\label{sec:cross-entropy-proof}
\input{sections/appendix_C_updated}

\section{Additional Experimental Details}
\label{sec:appendix-experimental-details}
\input{sections/appendix_D_experiment_details}

\newpage

\section{Additional Experimental Results}
\label{sec:appendix-experimental-results}
\input{sections/appendix_E_experiment_results}

\newpage

~\newpage

\section{Discussion of Convex Polytope Geometry}
\label{sec:appendix-convex-polytope-geometry}
\input{sections/appendix_F_geometry}

\section{Ablation Study on Number of Clusters}
\label{sec:ablation}

\input{sections/appendix_G_ablation}

\newpage

\section{Ablation Study with a Naïve Feature Space}
\label{sec:ablation-naive}
\input{sections/appendix_H_ablation_naive}

\end{document}

%% file: sections/00_abstract.tex
What sorts of structure might enable a learner to discover classes from unlabeled data? Traditional approaches rely on feature-space similarity and heroic assumptions on the data. In this paper, we introduce unsupervised learning under $\emph{Latent Label Shift}$ (LLS), where we have access to unlabeled data from multiple domains such that the label marginals $p_d(y)$  \update{can shift across domains} but the class conditionals $p(\mathbf{x}|y)$ do not. This work instantiates a new principle for identifying classes: elements that shift together group together. For finite input spaces, we establish an isomorphism between LLS and topic modeling: inputs correspond to words, domains to documents, and labels to topics. Addressing continuous data, we prove that when each label's support contains a separable region, analogous to an anchor word, oracle access to $p(d|\mathbf{x})$ suffices to identify $p_d(y)$ and $p_d(y|\mathbf{x})$ up to permutation. Thus motivated, we introduce a practical algorithm that leverages domain-discriminative models as follows: (i) push examples through domain discriminator $p(d|\mathbf{x})$; (ii) discretize the data by clustering examples in $p(d|\mathbf{x})$ space; (iii) perform non-negative matrix factorization on the discrete data; (iv) combine the recovered $p(y|d)$ with the discriminator outputs $p(d|\mathbf{x})$ to compute $p_d(y|x) \; \forall d$. With semi-synthetic experiments, we show that our algorithm can leverage domain information to improve
upon competitive unsupervised classification methods. We reveal a failure mode of standard unsupervised classification methods when data-space similarity does not indicate true groupings, and show empirically that our method better handles this case. Our results establish a deep connection between distribution shift and topic modeling, opening promising lines for future work\footnote{Code is available at \url{https://github.com/acmi-lab/Latent-Label-Shift-DDFA}.}.

%% file: sections/01_intro.tex
Discovering systems of categories from unlabeled data
is a fundamental but ill-posed challenge in machine learning.
Typical unsupervised learning methods group instances 
together based on feature-space similarity.
Accordingly, given a collection of photographs of animals,
a practitioner might hope that, 
in some appropriate feature space,
images of animals of the same species
should be somehow similar to each other.
But why should we expect a clustering algorithm 
to recognize that dogs viewed in sunlight
and dogs viewed at night belong to the same category?
Why should we expect that butterflies and caterpillars
should lie close together in feature space?

In this paper, we offer an alternative principle
according to which we might identify a set of classes:
we exploit distribution shift across times and locations
to reveal otherwise unrecognizable groupings among examples.
For example, if we noticed that whenever we found ourselves 
in a location where butterflies are abundant,
caterpillars were similarly abundant,
and that whenever butterflies were scarce, 
caterpillars had a similar drop in prevalence,
we might conclude that the two were tied
to the same underlying concept, no matter 
how different they appear in feature space. 
In short, our principle suggests that latent classes might be uncovered whenever 
\emph{instances that shift together group together}.

Formalizing this intuition, we introduce the problem 
of unsupervised learning under \emph{Latent Label Shift} (LLS). 
Here, we assume access to a collection of domains $d \in \{1,\dots, r\}$,
where the mixture proportions $p_d(y)$ vary across domains
but the class conditional distribution $p(x|y)$ is domain-invariant. 
Our goals are to recover the underlying classes up to permutation,
and thus to identify both the per-domain mixture proportions $p_d(y)$
and optimally adapted per-domain classifiers $p_d(y|x)$.
The essential feature of our setup is that
only the true $y$'s, as characterized 
by their class-conditional distributions $p(x|y)$,
could account for the observed shifts in $p_d(x)$.
We prove that under mild assumptions,
knowledge of this underlying structure is sufficient 
for inducing the full set of categories.

\begin{figure}[t]
  \label{fig:Pipeline}
  \centering
  \includegraphics[width=\textwidth]{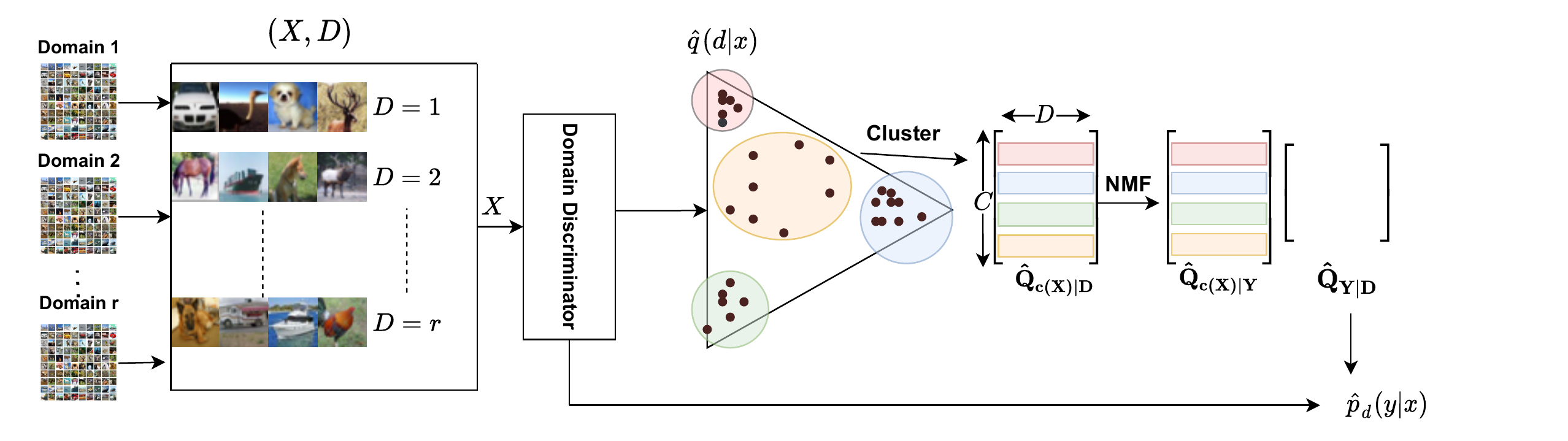}
  \caption{\textbf{Schematic of our DDFA algorithm}.
  After training a domain discriminator, 
  we (i) push all data through the discriminator;
  (ii) cluster the data based on discriminator outputs;
  (iii) solve the resulting discrete topic modeling problem and then combine $\hat{q}(d|x)$ and $\hat{q}(y, d)$ to estimate $\hat{p}_d(y|x)$.
}

\end{figure}

First, we focus on the \emph{tabular setting},
demonstrating that when the input space 
is discrete and finite,
LLS is isomorphic to topic modeling \citep{blei2003latent}. 
Here, each distinct input $x$ maps to a \emph{word}
each \emph{latent label} $y$ maps to a \emph{topic}
and each domain $d$ maps to a document. 
In this case, we can apply standard identification results
for topic modeling \citep{donoho,arora-svd,SPA,huang2016anchor,chen2021learning}
that rely only on the existence 
of anchor words within each topic
(\update{i.e.,} for each label \update{$y$} there is at least one $x$
in the support of \update{$y$},
that is not in the support of any \update{$y' \neq y$}).
Here, standard methods based on Non-negative Matrix Factorization (NMF)
can recover each domain's underlying 
mixture proportion $p_d(y)$
and optimal predictor $p_d(y|x)$~\citep{donoho, huang2016anchor, SPA}.
However, the restriction to discrete inputs,
while appropriate for topic modeling,
proves restrictive when our interests 
extend to high-dimensional continuous input spaces.

Then, to handle high-dimensional inputs, 
we propose \emph{Discriminate-Discretize-Factorize-Adjust (DDFA)},
a general framework that proceeds in the following steps:
(i) pool data from all domains to produce a mixture distribution $q(x,d)$;
(ii) train a domain discriminative model $f$ to predict $q(d|x)$;
(iii) push all data through $f$,
cluster examples in the pushforward distribution,
and tabularize the data based on cluster membership;
(iv) solve the resulting discrete topic modeling problem (e.g., via NMF), 
estimating $q(y,d)$ up to permutation
of the latent labels;
(v) combine the predicted $q(d|x)$
and $q(y,d)$ 
\update{to} estimate $p_d(y)$
and
$p_d(y|x)$.
In developing this approach, we draw inspiration 
from recent works on distribution shift 
and learning from positive and unlabeled data
that (i) leverage black box predictors 
to perform dimensionality reduction
\citep{lipton2018detecting, garg2020unified, garg2021mixture};
and (ii) work with \emph{anchor sets}, 
separable subsets of continuous input spaces 
that belong to only one class's support
\citep{scott15, Liu_2016, du_Plessis_2016, Bekker_Davis_2018, garg2021mixture}.

Our \update{key} theoretical result shows that domain discrimination \update{($q(d|x)$)}
provides a sufficient representation for identifying all parameters of interest. Given oracle access to $q(d|x)$
(which is identified without labels),
our procedure is
\update{consistent}. 
Our analysis reveals that the true $q(d|x)$
maps all points in the same anchor set 
to a single point mass in the push-forward distribution. 
This motivates our practical approach of discretizing data 
by hunting for tight clusters in \update{${q}(d|x)$} space.

In semi-synthetic experiments,
we adapt existing image classification benchmarks to the LLS setting,
sampling without replacement to construct
collections of label-shifted domains.
We note that training a domain discriminator classifier is a difficult task,
and find that warm starting the initial layers of our model
with pretrained weights from unsupervised approaches 
can significantly boost performance.
We show that warm-started DDFA outperforms
competitive unsupervised approaches on CIFAR-10 and CIFAR-20 when domain marginals $p_d(y)$ are sufficiently sparse.
In particular, we observe improvements of as much as $30\%$ accuracy
over a recent high-performing unsupervised method on CIFAR-20.
Further, on subsets of FieldGuide dataset, 
where similarity between species and diversity 
within a species leads to failure of unsupervised learning, 
we show that DDFA recovers the true distinctions.
To be clear, these are not apples-to-apples comparisons:
our methods are specifically tailored to the LLS setting.
The takeaway is that the structure of the \emph{LLS} setting
can be exploited to outperform the best unsupervised learning heuristics.

%% file: sections/02_related.tex
\textbf{Unsupervised Learning {} {}}
Standard unsupervised learning approaches for discovering labels 
often rely on similarity in the original data space \citep{kmeans,gmm}.
While distances in feature space become
meaningless for high-dimensional data,
deep learning researchers have turned 
to similarity in a representation space
learned via self-supervised contrastive tasks 
\citep{jigsaw,contrastive-context, contrastive-rotations,chen2020simple},
or similarity in a feature space learned end-to-end for a clustering task
\citep{deepclustering, deepercluster, RUC, SCAN}.
Our problem setup closely resembles independent component analysis (ICA), 
where one seeks to identify statistically independent signal components from mixtures \citep{ICA}. 
However, ICA's assumption of statistical independence among the components does not generally hold in our setup.
In topic modeling \citep{blei2003latent,arora-svd,huang2016anchor,chen2021learning,papa},
documents are modeled as mixtures of topics, and topics 
as categorical distributions over a finite vocabulary.
\update{Early topic models include the well-known Latent Dirichlet Allocation (LDA) \citep{blei2003latent}, which assumes that topic mixing coefficients are drawn from a Dirichlet distribution,
along with papers with more relaxed assumptions 
on the distribution of topic mixing coefficients (e.g., pLSI) \cite{pLSI,papa}}. 
The topic modeling literature often draws on
Non-negative Matrix Factorization (NMF) methods 
\citep{paatero1994positive,NMFAlgos},
which decompose a given matrix into a product 
of two matrices with non-negative elements
\citep{pLSI-equals-NMF-ding,de2016equivalence, plsi-equals-nmf-gaussier,girolami2003equivalence}. 
In both Topic Modeling and NMF, 
a fundamental problem has been to characterize the precise conditions 
under which the system is uniquely identifiable 
\citep{donoho,arora-svd,huang2016anchor,chen2021learning}.
The anchor condition (also referred to as separability)
is known to be instrumental for identifying topic models
\citep{arora-svd,chen2021learning,huang2016anchor,donoho}.
In this work, we extend these ideas, leveraging separable subsets 
of each label's support (the anchor sets) to produce 
anchor words in the discretized problem. 
Existing methods have attempted to extend latent variable modeling
to continuous input domains by making assumptions 
about the functional forms of the class-conditional densities, 
e.g., restricting to Gaussian mixtures \citep{gmm,gmmLDA}.
A second line of approach involves finding an appropriate discretization of the continuous space \citep{ideafordisc}.

\textbf{Distribution Shift under
the Label Shift Assumption {} {}}
The label shift assumption, 
where $p_d(y)$ can vary but $p(x|y)$ cannot, has been extensively studied 
in the domain adaptation literature
\citep{saerens2002adjusting,storkey2009training,zhang2013domain,lipton2018detecting,guo2020ltf,garg2020unified}
and also \update{features} in the problem of learning 
from positive and unlabeled data 
\citep{elkan2008learning, pusurvey, garg2021mixture}.
For both problems, many classical approaches 
suffer from the curse of dimensionality,
failing in the settings where deep learning prevails. 
Our solution strategy draws inspiration from recent work
on label shift \citep{lipton2018detecting, alexandari2019adapting, azizzadenesheli2019regularized,  garg2020unified} 
and PU learning \citep{pusurvey,Liu_2016,scott15, garg2021mixture, garg2022OSLS} that leverage black-box predictors to produce 
sufficient low-dimensional representations 
for identifying target distributions of interest
(other works leverage black box predictors 
heuristically \citep{ivanov2019dedpul}).
\textbf{Key differences:}
While PU learning requires identifying
\emph{one} new class for which we lack labeled examples
provided that the positive class contains an anchor set \citep{garg2021mixture},
LLS can identify an arbitrary number of classes (up to permutation)
from completely unlabeled data, provided a sufficient number of domains.

\textbf{Domain Generalization {} {}}
The related problem of Domain Generalization (DG)
also addresses learning with data 
drawn from multiple distributions 
and where the domain identifiers 
play a key role~\citep{muandet2013domain, arjovsky2019invariant}.
However in DG, we are given 
\emph{labeled} data from multiple domains, 
and our goal is to learn a classifier 
that can generalize to new domains. 
By contrast, in LLS, we work with unlabeled data only,
leveraging the problem structure 
to identify the underlying labels.

%% file: sections/03_LLS.tex
\textbf{Notation {} {}} For a vector $v\in \Real^p$, 
we use $v_j$ to denote its $j^\text{th}$ entry, 
and for an event $E$, we let $\indict{E}$ 
denote the binary indicator of the event.
By $\abs{A}$, we denote the cardinality of set $A$. 
With $[n]$, we denote the set $\{1,2,\ldots, n\}$. 
We use $[A]_{i,j}$ to access the element at 
$(i,j)$ in $A$.
Let $\inpt$ be the input space 
and $\out = \{ 1,2, \ldots, k\}$ \update{be} the output space
for multiclass classification. 
\update{We assume throughout this work that the number of true classes $k$ is known.}
Throughout this paper, we use capital letters 
to denote random variables and small case letters 
to denote the corresponding values they take. 
For example, by $X$ we denote the input random variable 
and by $x$, we denote a value that $X$ may take.

We now formally introduce the problem of 
unsupervised learning under LLS.
In LLS, we assume that we observe 
unlabeled data from $r$ domains.
Let $\calR = \{1, 2, \ldots, r\}$ 
be the set of domains. 
By $p_d$, we denote the probability density
(or mass) function for each domain $d\in \calR$.

\begin{definition}[Latent label shift] \label{def:latent_label_shift}
We observe data from $r$ domains. 
While the label distribution 
\update{can differ across the domains}, 
for all $d, d' \in \calR$
and for all $(x, y) \in \inpt \times \out$, 
we have $p_d(x| y) = p_{d'}(x|y) $.
\end{definition}

Simply put, \defref{def:latent_label_shift} states 
that the conditional distribution $p_d(x|y)$
remains invariant across domains,
i.e., they satisfy the label shift assumption. 
Thus, we can drop the subscript on this factor,
denoting all $p_d(x|y)$ by $p(x|y)$. 
Crucially, under LLS, $p_d(y)$ 
can vary across different domains. 
Under LLS, we observe unlabeled data 
with domain label 
$\{(x_1, d_1), (x_2, d_2), \ldots, (x_n, d_n)\}$. 
Our goal breaks down into two tasks.
\update{Up to} permutation of labels, we aim to 
(i) estimate the label marginal in each domain $p_d(y)$; 
and (ii) estimate the optimal per-domain predictor $p_d(y|x)$.

\begin{figure}[t!]
    \centering
    \begin{tikzpicture}
    
    \node[obs]      (D)                              {D};
    \node[latent]        (Y)       [right=of D] {Y};
    \node[obs]      (X)       [right=of Y] {X};
    
    \edge [] {D} {Y}
    \edge [] {Y} {X}
    \end{tikzpicture}    
    \caption{Relationship under $Q$ between observed $D$, observed $X$, and latent $Y$.}
    \label{fig:pgm_dyx}
\end{figure}
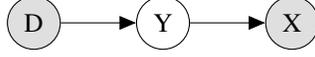

\textbf{Mixing distribution Q{} {}}
A key step in our algorithm will be 
to train a domain discriminative model.
Towards this end we define $Q$,
a distribution over $\calX\times\calY\times\calR$,
constructed by taking a uniform mixture 
over all domains.
By $q$, we denote the probability density
(or mass) function of $Q$.
Define $Q$ such that 
$q(x, y | D = d) = p_d(x,y)$, 
i.e., when we condition on $D=d$ 
we recover the joint distribution 
over $\calX \times \calY$ 
specific to that domain $d$.
For all $d \in \calR$, 
we define $\gamma_d = q(d)$, 
i.e., the prevalence of each domain
in our distribution $Q$.  
Notice that $q(x,y)$ is a mixture over 
the distributions $\{ p_d (x,y) \}_{d\in \calR}$, 
with $\{ \gamma_d \}_{d \in \calR}$ 
as the corresponding mixture coefficients.
Under LLS (\defref{def:latent_label_shift}), 
$X$ does not depend on $D$ when conditioned on $Y$
(\figref{fig:pgm_dyx}).

\textbf{Additional notation for the discrete case {} {}}
To begin, we setup notation for 
discrete input spaces with $\abs{\inpt} = m$.
Without loss of generality, 
we assume that $\calX = \{1, 2, \ldots, m\}$.
The label shift assumption allows us 
to formulate the label marginal 
estimation problem in matrix form. 
Let $\bQ_{X|D}$ be an $m \times r$ matrix
such that 
$[\bQ_{X|D}]_{i,d} = p_d(X = i)$, i.e., 
the $d$-{th} column of $\bQ_{X|D}$ is $p_d(x)$.
Let $\bQ_{X|Y}$ be an 
$m \times k$ matrix such 
that $[\bQ_{X|Y}]_{i,j} = p(X=i|Y=j)$, 
the $j$-th column is a distribution over $X$ given $Y=j$.
Similarly, define $\bQ_{Y|D}$ as a $k\times r$ matrix 
whose $d$-th column is the domain marginal $p_d(y)$. 
Now with \defref{def:latent_label_shift}, we have 
$p_d(x) = \sum_{y} p_d(x,y) = \sum_{y} p_d(x|y)p_d(y) = \sum_{y} p(x|y)p_d(y) $. 
Since this is true $\forall d \in \calR$, 
we can express this in a matrix form as 
$\bQ_{X|D} = \bQ_{X|Y}\bQ_{Y|D}$. 

\textbf{Additional assumptions {} {}} 
Before we present identifiability results for the LLS problem, 
we introduce four additional assumptions required throughout the paper: 

\begin{enumerate}[label=A.\arabic*]
    \item There are at least as many domains as classes, i.e., $\abs{\calR} \ge \abs{\out}$. 
    \item The matrix formed by label marginals (as columns) across different domains is full-rank, i.e., $\textrm{rank}({\bQ}_{Y|D}) = k$.
    \item Equal representation of domains, i.e., 
    for all $ d \in \calR, \gamma_d = \nicefrac{1}{r}$.
    \item Fix $\epsilon > 0$. For all $y \in \out$, 
    there exists a \update{unique} subdomain $A_y \subseteq \inpt$, 
    such that $q(A_y) \ge \epsilon$ \update{ and $x \in A_y$ if and only if the following conditions are satisfied: $q(x|y) > 0$ and $q(x|y') = 0 $ for all $y^\prime \in \out \setminus \{y\}$}.
    We refer to this assumption as the $\epsilon$-anchor sub-domain condition. 

\end{enumerate}

We now comment on the assumptions. 
A.1--A.2 are benign, these assumptions just imply 
that the matrix \update{${\bf Q}_{Y|D}$} is full row rank. 
Without loss of generality, A.3 can be assumed 
when dealing with data from a collection of domains. 
When this condition is not satisfied, 
one could just re-sample data points 
uniformly at random from each domain $d$. 
Intuitively, A.4 states that for each label $y \in \out$,
we have some subset of inputs that only belong to that class $y$. 
To avoid vanishing probability of this subset, 
we ensure at least $\epsilon$ probability mass 
in our mixing distribution $Q$. 
The anchor word  condition is related 
to the positive sub-domain in PU learning, 
which requires that there exists a subset of $\mathcal{X}$ 
in which all examples only belong 
to the positive class~\citep{scott15, Liu_2016, du_Plessis_2016, Bekker_Davis_2018}.

%% file: sections/04_theory.tex
In this section, we establish identifiability of LLS problem. 
We begin by considering the case where the input space 
is discrete and formalize the isomorphism to topic modeling.
Then we establish the identifiability 
of the system in this discrete setting 
by appealing to existing results 
in topic modeling~\citep{huang2016anchor}. 
Finally, extending results from discrete case, 
we provide novel analysis to establish
our identifiability result for the continuous setting. 

\textbf{Isomorphism to topic modeling {} {}}
Recall that for the discrete input setting, 
we have the matrix formulation: 
$\bQ_{X|D} = \bQ_{X|Y}\bQ_{Y|D}$. 
Consider a corpus of $r$ documents, 
consisting of terms from a vocabulary of size $m$.
Let $\doc$ be an $\mathbb{R}^{m \times r}$ matrix 
representing the underlying corpus. 
Each column of $\doc$ represents a document,
and each row represents a term in the vocabulary. 
Each element $[\doc]_{i,j}$ represents 
the frequency of term $i$ in document $j$. 
Topic modeling~\citep{blei2003latent,pLSI,huang2016anchor,arora-svd} 
considers each document to be composed as a mixture of $k$ topics.
Each topic prescribes a frequency with which
the terms in the vocabulary occur given that topic. 
Further, the proportion of each topic varies 
across documents with the frequency of terms 
given topic remaining invariant.

We can state the topic \update{modeling} problem as:
$\doc = \termtopic \topictopic$,
where $\termtopic$ is an $\mathbb{R}^{m \times k }$ matrix,
$[\termtopic]_{i,j}$ represents the frequency of term $i$ given topic $j$,
and $\topictopic$ is an $\mathbb{R}^{k \times r}$ matrix,
where $[\topictopic]_{i,j}$ represents 
the proportion of topic $i$ in document $j$. 
Note that all three matrices are column normalized. 
The isomorphism is then between document and domain,
topic and label, term and input sample, i.e., 
$\doc = \termtopic\topictopic \equiv \bQ_{X|D} = \bQ_{X|Y}\bQ_{Y|D}$. 
In both the cases, we are interested in decomposing 
a known matrix into two unknown matrices. 
This formulation is examined as a non-negative matrix factorization problem
with an added simplicial constraint on the columns (columns sum to 1) \citep{arora-nmf-provably,SPA}. 

Identifiability of the topic modeling problem is well-established \citep{donoho,arora-svd,SPA,huang2016anchor,chen2021learning}.
We leverage the isomorphism to topic modeling 
to extend this identifiability condition to our LLS setting.
We formalize the \update{adaptation} here:

\begin{theorem}
\label{thm:separability-discrete}
(adapted from Proposition 1 in \citet{huang2016anchor})
Assume A.1, A.2 and A.4 hold 
(A.4 in the discrete setting is referred to as the anchor word condition).
Then the solution to $\bQ_{X|D} = \bQ_{X|Y}\bQ_{Y|D}$ is uniquely identified \update{up to permutation of class labels}. 
\end{theorem}

We refer readers to \citet{huang2016anchor} for a proof of this theorem. 
Intuitively, \thmref{thm:separability-discrete} states 
that if each label $y$ has at least one token in the input space 
that has support only in $y$, and A.1, A.2 hold, 
then the solution to $ \bQ_{X|Y}$, $\bQ_{Y|D}$ is unique. 
Furthermore, under this condition, there exist algorithms 
that can recover $\bQ_{X|Y}$, $\bQ_{Y|D}$
within some permutation \update{of class labels} \citep{huang2016anchor,SPA,arora-nmf-provably,arora-svd}. 

\textbf{Extensions to the continuous case {} {}} 
We will prove identifiability in the continuous setting, 
when $\mathcal{X} = \mathbb{R}^p$ for some $p \ge 1$. 
In addition to A.1--A.4, we make an additional assumption 
that we have oracle access to $q(d|x)$, i.e., 
the true domain discriminator for mixture distribution $Q$. 
This is implied by assuming access 
to the marginal $q(x,d)$ 
from which we observe our samples.
Formally, we define a push forward function $f$ 
such that $[f(x)]_d = q(d|x)$, then push the data forward through  $f$ to obtain outputs in $\Delta^{r-1}$. In the proof of \thmref{thm:anchor-subdomain-continuous}, we will show that these outputs can be discretized in a fashion that maps anchor subdomains to anchor words in a tabular, discrete setting.
We separately remark that the anchor word outputs are in fact extreme corners of the convex polytope in $\Delta^{r-1}$ which encloses all $f(x)$ mass; we discuss this geometry further in \appref{sec:appendix-convex-polytope-geometry}.
After constructing the anchor word discretization,
we appeal to \thmref{thm:separability-discrete}
to recover $\bQ_{Y|D}$. 
Given  $\bQ_{Y|D}$, 
we show that we can use Bayes' rule and the 
LLS condition (\defref{def:latent_label_shift})
to identify the distribution $q(y|x,d) = p_d(y|x)$
over latent variable $y$.
We formalize this in the following theorem:

\begin{theorem}
\label{thm:anchor-subdomain-continuous}
Let the distribution $Q$ over random variables
$X,Y,D$ satisfy Assumptions A.1--A.4. 
Assuming access to the joint distribution $q(x, d)$, \update{and knowledge of the number of true classes $k$}, 
we show that the following quantities are identifiable: 
(i) $\bQ_{Y|D}$, (ii) $q(y|X = x)\,$,
for all  $x \in \inpt$ that lies 
in the support (i.e. $q(x) > 0$); 
and (iii) $q(y|X = x,D = d)\,$, 
for all $x \in \inpt$ and $d \in \calR$
such that  $q(x,d) > 0$. 
\end{theorem}

Before presenting a proof sketch for \thmref{thm:anchor-subdomain-continuous}, 
we first present key lemmas 
(we include their proofs in \appref{sec:lemma-proofs}).

\begin{lemma}
\label{lemma:solve-domain-agnostic}
Under the same assumptions as \thmref{thm:anchor-subdomain-continuous},
the matrix $\bQ_{Y|D}$ and $f(x) = q(d|x)$ 
uniquely determine $q(y|x)$ for all $y \in \out$ 
and $x \in \mathcal{X}$ such that $q(x) > 0$.
\end{lemma}

\lemref{lemma:solve-domain-agnostic} states 
that given matrix $\bQ_{Y|D}$ and oracle domain discriminator, 
we can uniquely identify $q(y|x)$. 
In particular, we show that for any $x \in \calX$,
$q(d|x)$ can be expressed as a convex combination 
of the $k$ columns of $\bQ_{D|Y}$ (which is computed from $\bQ_{Y|D}$ and is column rank $k$)
and the coefficients of the combination are $q(y|x)$.
Combining this with the linear independence
of the columns of $\bQ_{D|Y}$, 
we show that these coefficients are unique. 
In the following lemma, we show how the identified $q(y|x)$ 
can then be used to identify $q(y|x,d)$: 

\begin{lemma}
\label{lemma:solve-domain-informed}
Under the same assumptions as \thmref{thm:anchor-subdomain-continuous},
for all $y \in \out$, and $x \in \mathcal{X}$ such that $q(x,d) > 0$.
the matrix $\bQ_{Y|D}\,$ and $q(y|x)$ 
uniquely determine $q(y|x,d)$.
\end{lemma}

To prove \lemref{lemma:solve-domain-informed}, 
we show that we can combine the conditional distribution 
over the labels given a sample $x \in \calX$
with the prior distribution of the labels in each domain
to determine the posterior distribution over labels 
given the sample $x$ and the domain of interest. 
Next, we introduce a key property 
of the domain discriminator classifier $f$:

\begin{lemma}
\label{lemma:dd_maps_to_unique}
Under the same assumptions as \thmref{thm:anchor-subdomain-continuous},
for all $x,x^\prime$ in anchor sub-domain, i.e., $ x, x^\prime \in A_y$ 
for a given label $y\in \out$, we have $f(x) = f(x^\prime)$.
Further, for any $y \in \out$, if $x \in A_y , x^\prime \notin A_y$, 
then $f(x) \neq f(x^\prime)$. 
\end{lemma}

\lemref{lemma:dd_maps_to_unique} implies that the oracle domain discriminator $f$
maps all points in an anchor subdomain,
and only those points in that anchor subdomain 
to the same point in $f(x) = q(d|x)$ space. 
We can now present a proof sketch 
for \thmref{thm:anchor-subdomain-continuous} 
(full proof in \appref{sec:theorem-2-proof_app}): 

\begin{proof}[Proof sketch of \thmref{thm:anchor-subdomain-continuous}]\renewcommand{\qedsymbol}{}
The key idea of the proof lies in
proposing a discretization such that 
some subset of anchor subdomains for each label $y$
in the continuous space map 
to distinct anchor words in discrete space.  
In particular, if there exists a discretization 
of the continuous space $\calX$ 
that for any $y \in \out$,
maps all $x \in A_y$ to the same point in the discrete space, 
but no $x \notin A_y$ maps to this point, 
then this point serves as an anchor word.
From Lemma \ref{lemma:dd_maps_to_unique}, 
we know that all the $x \in A_y$ 
and only the $x \in A_y$ get mapped 
to specific points in the $f(x)$ space. 
Pushing all the $x \in \calX$ through $f$, we know from A.4 that there exists $k$ point masses of size $\epsilon$, one for each $f(A_y)$ in the $f(x) = q(d|x)$ space. 
We can now inspect this space for point masses of size at least $\epsilon$ to find at most $\mathcal O(\nicefrac{1}{\epsilon})$ such point masses among which are contained the $k$ point masses corresponding to the anchor subdomains. 
Discretizing this space by assigning each point mass to a group (and non-point masses to a single additional group), we have $k$ groups that have support only in one $y$ each. Thus, we have achieved a discretization with anchor words. Further, since the discrete space arises from a pushforward of the continuous space through $f$, the discrete space also satisfies the latent label shift assumption A.1. We now use \thmref{thm:separability-discrete} to claim identifiability of $\bQ_{Y|D}$. We then use Lemmas \ref{lemma:solve-domain-agnostic} and \ref{lemma:solve-domain-informed} to prove parts (ii) and (iii).
\end{proof}

%% file: sections/05_PracticalAlgorithm.tex
Motivated by our identifiability analysis, in this section, we present an algorithm to estimate $\bQ_{Y|D}, q(y| x)$, and $q(y |x, d)$ when $X$ is continuous by exploiting domain structure and approximating the true domain discriminator $f$. Intuitively, $q(y|x,d)$ is the domain specific classifier  $p_d(y|x)$ and $q(y|x)$ is the classifier for data from aggregated domains. $\bQ_{Y|D}$ captures label marginal 
for individual domains. 
A naive approach would be to aggregate data from different domains and 
exploit recent advancements in unsupervised learning~\citep{SCAN, RUC, deepclustering, deepercluster}.
However, aggregating data from multiple domains 
loses the domain structure that we hope to leverage. 
We highlight this 
failure mode of the 
\update{traditional} unsupervised clustering method in \secref{sec:exp}. \update{We remark that DDFA draws heavy inspiration from the proof of \thmref{thm:anchor-subdomain-continuous}, but we do not present a guarantee that the DDFA solution will converge to the identifiable solution. This is primarily due to the K-means clustering heuristic we rely on, which empirically offers effective noise tolerance, but theoretically has no guarantee of correct convergence}.

\textbf{Discriminate {} {}} We begin Algorithm \ref{alg:1} by creating a split of the unlabeled samples into the training and validation sets. Using the unlabeled data samples and the domain that each sample originated from, we first train a domain discriminative classifier $\smash{\hat{f}}$. The domain discriminative classifier outputs a distribution over domains for a given input. This classifier is trained with cross-entropy loss to predict the domain label of each sample on the training set. With unlimited data, the minimizer of this loss is the true $f$, as we prove in \appref{sec:cross-entropy-proof}.
To avoid overfitting, we stop training $\smash{\hat{f}}$ when the cross-entropy loss on the validation set stops decreasing. 
Note that here the validation set also only contains domain \update{indices} (and no information about true \update{class} labels). 

\textbf{Discretize {} {}} We now push forward all the samples from the training and validation sets through the domain discriminator to get vector $\smash{\hat{f}(x_i)}$ for each sample $x_i$. In the proof of  \thmref{thm:anchor-subdomain-continuous}, we argue that when working with true $f$, and the entire marginal \update{$q(x,d)$}, we can choose a discretization satisfying the anchor word assumption by identifying point masses in the distribution of $f(x)$ and filtering to include those of at least $\epsilon$ \update{mass}. In the practical setting, because we have only a finite set of data points and a noisy $\smash{\hat{f}}$, we use clustering to approximately find point masses.
We choose $m \ge k$ and recover $m$ clusters with any standard clustering procedure (e.g. K-means).
\update{This clustering procedure is effectively a useful, but imperfect heuristic: if the noise in $\smash{\hat{f}}$ is sufficiently small and the clustering sufficiently granular, we hope that our $m$ discovered clusters will include $k$ \emph{pure} clusters, each of which only contains data points from a different anchor subdomain which are tightly packed around the true $f(A_y)$ for the corresponding label $y$. Clustering in this space is superior to a naive clustering on the input space because close proximity in this space indicates similarity in $q(d|x)$.}

Let us denote the learned clustering function as $c$, where $c(x)$ is the cluster assigned to a datapoint $x$. 
We now leverage the cluster id $c(x_i)$ of each sample $x_i$ to discretize samples into a finite discrete space $[m]$. Combining the cluster id with the domain source $d_i$ for each sample, we estimate $\smash{\hat{\bQ}_{c(X)|D}}$ by simply computing, for each domain, the fraction of its samples assigned to each cluster.

\textbf{Factorize {} {}} We apply an NMF algorithm to $\hat{\bQ}_{c(X)|D}$ to obtain estimates of $\hat{\bQ}_{c(X)|Y}$ and $\hat{\bQ}_{Y|D}$.

\textbf{Adjust {} {}}
We begin Algorithm \ref{alg:2} by considering a test point $(x',d')$. To make a prediction, if we had access to oracle $f$ and true $\bQ_{Y|D}$, we could precisely compute $q(y| x')$ (\lemref{lemma:solve-domain-agnostic}).
However, in place of these true quantities, we plug in the estimates $\smash{\hat{f}}$ and $\smash{\hat{\bQ}_{Y|D}}$. Since our estimates contain noise, the estimate $\smash{\hat{q}(y|x')}$ is found by left-multiplying ${\hat{f}(x')}$ with the pseudo-inverse of ${\hat{\bQ}_{D|Y}}$, as opposed to solving a sufficient system of equations.
As our estimates $\hat{f}$ and $\hat{\bQ}_{D|Y}$ approach the true values, the projection of $\hat{f}(x')$ into the column space of $\hat{\bQ}_{D|Y}$ tends to $\hat{f}(x')$ itself, so the pseudo-inverse approaches the true solution.
Now we can use the constructive procedure introduced in the proof of \lemref{lemma:solve-domain-informed} to
compute the plug-in estimate $\hat{q}(y|x',d') = \hat{p}_{d'}(y|x')$.

\setlength{\textfloatsep}{12pt}
\begin{algorithm}[t]
\caption{DDFA Training}
\label{alg:1}
\begin{algorithmic}[1]

\INPUT $k \geq 1, r \geq k, \{(x_i,d_i)\}_{i \in [n]} \sim q(x, d), \textrm{ A class of functions } \mathcal{F} \textrm{ from } \mathbb{R}^p \to \mathbb{R}^r $

\STATE \textrm{Split into train set $T$ and validation set $V$ } 

\STATE $\textrm{Train } \hat{f} \in \mathcal{F} \textrm{ to minimize cross entropy loss for predicting } d|x \textrm{ on }T\textrm{ with early stopping on }V $

\STATE $\textrm{Push all } \{x_i\}_{i \in [n]} \textrm{ through } \hat{f} $ 

\STATE $\textrm{Train clustering algorithm on the n points } \{\hat{f}(x_i)\}_{i \in [n]}, \textrm{ obtain }m \textrm{ clusters.}$

\STATE $c(x_i) \gets \textrm{Cluster id of } \hat{f}(x_i)$

\STATE $\hat{q}(c(X) = a | D = b) \gets  \frac{\sum_{i \in [n]} \mathbb{I}[c(x_i) = a, \; d_i = b ]}{\sum_{j \in [n]}\mathbb{I}[ d_j = b ]}$

\STATE $ \textrm{Populate } \hat{\bQ}_{c(X)|D} \textrm{ as } [\hat{\bQ}_{c(X)|D}]_{a, b} \gets \hat{q}(c(X) = a | D = b)$

\STATE $\hat{\bQ}_{c(X)|\update{Y}} , \hat{\bQ}_{Y|D} \gets \textrm{ NMF }(\hat{\bQ}_{c(X)|D} )$

\OUTPUT {$\hat{\bQ}_{Y|D}, \hat{f}$}

\end{algorithmic}

\end{algorithm}

\begin{algorithm}[t!]
\caption{DDFA Prediction}
\label{alg:2} 
\begin{algorithmic}[1]

\INPUT $\hat{\bQ}_{Y|D}, \hat{f}, (x', d') \sim q(x,d)$ 

\STATE $\textrm{Populate }\hat{\bQ}_{D|Y} \textrm{ as } [ \hat{\bQ}_{D|Y}]_{d,y} \gets \frac{[\hat{\bQ}_{Y|D}]_{y,d}}{\sum_{d''=1}^{d''=r} [\hat{\bQ}_{Y|D}]_{y,d''}}$

\STATE $\textrm{Assign } \hat{q}( y | X = x') \gets \left[\left(\hat{\bQ}_{D|Y}\right)^\dagger \hat{f}(x') \right]_y $

\STATE $\textrm{Assign }\hat{q}(y|X = x', D = d') \gets \cfrac{[\hat{\bQ}_{D|Y}]_{d',y}\hat{q}(y|X = x')}{ \sum\limits_{y'' \in [k]} [\hat{\bQ}_{D|Y}]_{d',y''}\hat{q}(y''|X = x') }$

\STATE $y_\textrm{pred} \gets \argmax_{y \in [k]} \hat{q}(y|X = x', D = d')$

\OUTPUT : { $\hat{q}(y|X = x', D = d') = \hat{p}_{d'}(y|x')$, $\hat{q}(y|X = x')$, $y_\textrm{pred}$ } 
\end{algorithmic}

\end{algorithm}

%% file: sections/06_ExperimentalSection.tex
Experiment code is available at 
\url{https://github.com/acmi-lab/Latent-Label-Shift-DDFA}.

\textbf{Baselines {} {}} We select the unsupervised classification method SCAN, as a high-performing competitive baseline \citep{SCAN}. SCAN pretrains a ResNet \citep{he2016deep} backbone using SimCLR \citep{chen2020simple} and MoCo \citep{moco} setups (pretext tasks). SCAN then trains a clustering head to minimize the SCAN loss (refer \citep{SCAN} for more details) \footnote{SCAN code: \url{https://github.com/wvangansbeke/Unsupervised-Classification}}. We do not use the SCAN self-labeling step. We make sure to evaluate SCAN on the same potentially class-imbalanced test subset we create for each experiment. Since SCAN is fit on a superset of the data DDFA sees, we believe this gives a slight data advantage to the SCAN baseline (although we acknowledge that the class balance for SCAN training is also potentially different from its evaluation class balance). To evaluate SCAN, we use the 
public pretrained weights
available for CIFAR-10, CIFAR-20, and ImageNet-50. We also train SCAN ourselves on the train and validation portions of the FieldGuide2 and FieldGuide28 datasets with a ResNet18 backbone and SimCLR pretext task. We replicate the hyperparameters used for CIFAR training.  

\textbf{Datasets {} {}}
First we examine 
standard multiclass image datasets CIFAR-10, CIFAR-20 \citep{CIFAR}, and ImageNet-50 \citep{imagenet} containing images from 10, 20, and 50 classes respectively. Images in these datasets typically focus on a single large object which dominates the center of the frame, so unsupervised classification methods which respond strongly to similarity in visual space are well-suited to recover true classes up to permutation. These datasets are often believed to be separable (i.e., single true label applies to each image), so every example falls in an anchor subdomain (satisfying A.4).

Motivated by the application of LLS problem, 
we consider the FieldGuide dataset \footnote{FieldGuide: \url{https://sites.google.com/view/fgvc6/competitions/butterflies-moths-2019}}, which contains images of moths and butterflies. The true classes in this dataset are species, but each class contains images taken in immature (caterpillar) and adult stages of life. Based on the intuition that butterflies from a given species look more like butterflies from other species than caterpillars from their own species, we hypothesize that unsupervised classification will learn incorrect class boundaries which distinguish caterpillars from butterflies, as opposed to recovering the true class boundaries. Due to high visual similarity between members of different classes, this dataset may indeed have slight overlap between classes. However, we hypothesize that anchor subdomain still holds, i.e., there exist some images from each class that could only come from that class.
Additionally, if we have access to data from multiple domains, it is natural to assume that 
within each domain the relative distribution of caterpillar to adult stages of each species stay relatively constant as compared to prevalence of different species. 
We create two subsets of this dataset: FieldGuide2, with two species, and FieldGuide28, with 28 species.

\textbf{LLS Setup {} {}} The full sampling procedure for semisynthetic experiments is described in \appref{sec:appendix-experimental-details}. Roughly, we sample $p_d(y)$ from a symmetric Dirichlet distribution with concentration $\nicefrac{\alpha}{k}$, \update{where $k$ is the number of classes and $\alpha$ is a generation parameter that adjusts the difficulty of the synthetic problem}, and enforce maximum condition number $\kappa$ on $\bf Q_{Y|D}$. Small $\alpha$ and small $\kappa$ encourages sparsity in $\bf Q_{Y|D}$, so each label tends to only appear in a few domains. Larger parameters encourages $p_d(y)$ to tend toward uniform. We draw from test, train, and valid datasets without replacement to match these distributions, but discard some examples due to class imbalance.

\textbf{Training and Evaluation {} {}}
The algorithm uses train and validation data consisting of pairs of images and domain indices. We train ResNet50 \citep{he2016deep} (with added dropout) on images $x_i$ with domain indices $d_i$ as the label, choose best iteration by valid loss, pass all training and validation data through $\hat{f}$, and cluster pushforward predictions $\smash[]{\hat{f}(x_i)}$ into $m \ge k$ clusters with Faiss K-Means \citep{johnson2019billion}. We compute the $\smash[]{\hat{\bQ}_{c(X)|D}}$ matrix and run NMF to obtain $\smash[]{\hat{\bQ}_{c(X)|Y}}$, $\smash[]{\hat{\bQ}_{Y|D}}$. To make columns sum to 1, we normalize columns of $\smash[]{\hat{\bQ}_{c(X)|Y}}$, multiply each column's normalization coefficient over the corresponding row of $\smash[]{\hat{\bQ}_{Y|D}}$ (to preserve correctness of the decomposition), and then normalize columns of $\smash[]{\hat{\bQ}_{Y|D}}$.
Some NMF algorithms only output solutions satisfying the anchor word property \citep{arora-nmf-provably,kumar2013fast, SPA}. We found the strict requirement of an exact anchor word solution to lead to low noise tolerance. We therefore use the Sklearn implementation of standard NMF \citep{SklearnNMFImpl-01, tan2012automatic, scikit-learn}.

We instantiate the domain discriminator as ResNet18, and preseed its backbone with SCAN \citep{SCAN} pre-trained weights or \citep{SCAN} contrastive pre-text weights. We denote these models DDFA~(SI) and DDFA~(SPI) respectively. We predict class labels with Algorithm $\ref{alg:2}$. With the Hungarian algorithm, implemented in \citep{crouse2016implementing, 2020SciPy-NMeth}, we compute the highest true accuracy among any permutation of these labels (denoted ``Test acc''). With the same permutation, we reorder rows of $\smash[]{\hat{P}_{Y|D}}$, then compute the average absolute difference between corresponding entries of $\smash[]{\hat{\bQ}_{Y|D}}$ and $\bQ_{Y|D} $ (denoted ``$\bQ_{Y|D}$ err'').

\begin{table}[t]
    \caption{\emph{Results on CIFAR-20}. Each entry is produced with the averaged result of 5 different random seeds. With DDFA (RI) we refer to DDFA with randomly initialized backbone.  With DDFA (SI) we refer to DDFA's backbone initialized with SCAN. Note that in DDFA (SI), we do not leverage SCAN for clustering. $\alpha$ is the Dirichlet parameter used for generating label marginals in each domain, $\kappa$ is the maximum allowed condition number of the generated $\bQ_{Y|D}$ matrix, $r$ is number of domains.}
    \vspace{5pt}
    \label{table:cifar20_main_paper}
    \centering
    \begin{tabular}{llllllll}
        \toprule
        \multirow{2}{*}{r} & \multirow{2}{*}{Approaches}    &\multicolumn{2}{c}{$\alpha: 0.5, \; \kappa: 8$} & \multicolumn{2}{c}{$\alpha: 3, \; \kappa: 12$} & \multicolumn{2}{c}{$\alpha: 10, \; \kappa: 20$} \\
        \cmidrule(lr){3-4} \cmidrule(lr){5-6} \cmidrule(lr){7-8}
         & & Test acc & $\bQ_{Y|D}$ err & Test acc & $\bQ_{Y|D}$ err & Test acc & $\bQ_{Y|D}$ err \\
\midrule
\multirow{2}{*}{20} &SCAN & 0.445  & 0.090  & 0.432  & 0.081  & \textbf{0.438} & 0.062  \\
&DDFA (RI) & 0.549  & 0.038  & 0.342  & 0.043  & 0.203  & 0.053  \\
&DDFA (SI) & \textbf{0.807} & \textbf{0.021} & \textbf{0.582} & \textbf{0.028} & 0.355  & \textbf{0.035} \\
\midrule
\multirow{2}{*}{25} &SCAN & 0.451  & 0.091  & 0.457  & 0.079  & 0.444  & 0.061  \\
&DDFA (RI) & 0.542  & 0.040  & 0.283  & 0.050  & 0.179  & 0.053  \\
&DDFA (SI) & \textbf{0.851} & \textbf{0.017} & \textbf{0.667} & \textbf{0.024} & \textbf{0.495} & \textbf{0.031} \\
\midrule
\multirow{2}{*}{30} &SCAN & 0.453  & 0.088  & 0.436  & 0.079  & 0.439  & 0.061  \\
&DDFA (RI) & 0.486  & 0.046  & 0.287  & 0.054  & 0.123  & 0.068  \\
&DDFA (SI) & \textbf{0.868} & \textbf{0.016} & \textbf{0.687} & \textbf{0.024} & \textbf{0.517} & \textbf{0.032} \\

        \bottomrule
    \end{tabular}
\end{table}

\begin{table}[t!]
    \caption{\emph{Results on FieldGuide-2}. \update{Each entry is produced with the averaged result of 5 different random seeds. With DDFA (RI) we refer to DDFA with randomly initialized backbone. With DDFA (SPI) we refer to DDFA initialized with pretext training adopted by SCAN. Note that in DDFA (SPI), we do not leverage SCAN for clustering. $\alpha$ is the Dirichlet parameter used for generating label marginals in each domain, $\kappa$ is the maximum allowed condition number of the generated $\bQ_{Y|D}$ matrix, $r$ is number of domains. Full results available in \appref{sec:appendix-experimental-results}}}
    \vspace{5pt}
    \label{table:fieldguide2-main-paper}
    \centering
    \begin{tabular}{llllllll}
        \toprule
        \multirow{2}{*}{r} & \multirow{2}{*}{Approaches}    &\multicolumn{2}{c}{$\alpha: 0.5, \; \kappa: 3$} & \multicolumn{2}{c}{$\alpha: 3, \; \kappa: 5$} & \multicolumn{2}{c}{$\alpha: 10, \; \kappa: 7$} \\
        \cmidrule(lr){3-4} \cmidrule(lr){5-6} \cmidrule(lr){7-8}
         & & Test acc & $\bQ_{Y|D}$ err & Test acc & $\bQ_{Y|D}$ err & Test acc & $\bQ_{Y|D}$ err \\

\midrule
\multirow{2}{*}{2} &SCAN & 0.589  & 0.880  & 0.591  & 0.372  & 0.601  & 0.283  \\
&DDFA (SPI) & \textbf{0.947} & \textbf{0.084} & \textbf{0.725} & \textbf{0.150} & \textbf{0.676} & \textbf{0.222} \\
\midrule
\multirow{2}{*}{5} &SCAN & 0.589  & 0.640  & 0.587  & 0.389  & 0.586  & \textbf{0.226} \\
&DDFA (SPI) & \textbf{0.857} & \textbf{0.181} & \textbf{0.738} & \textbf{0.194} & \textbf{0.610} & 0.253  \\
\midrule
\multirow{2}{*}{10} &SCAN & 0.587  & 0.716  & 0.589  & 0.347  & 0.590  & 0.193  \\
&DDFA (SPI) & \textbf{0.886} & \textbf{0.162} & \textbf{0.744} & \textbf{0.194} & \textbf{0.606} & \textbf{0.172} \\

        \bottomrule
    \end{tabular}
\end{table}

\textbf{Results {} {}}
On CIFAR-10, we observe that DDFA alone is incapable of matching highly competitive baseline SCAN performance---however, in suitably sparse problem settings (small $\alpha$), it comes substantially close, indicating good recovery of true classes. Due to space constraints, we include CIFAR-10 results in \appref{sec:appendix-experimental-results}.
DDFA (SI) combines SCAN's strong pretrain with domain discrimination fine-tuning to outperform SCAN in the easiest, sparsest setting and \emph{certain} denser settings. On CIFAR-20, baseline SCAN is much less competitive, so our DDFA(SI) dominates baseline SCAN in all settings except the densest (\tabref{table:cifar20_main_paper}). These results demonstrate how adding domain information can dramatically boost unsupervised baselines.

On FieldGuide-2, DDFA (SPI) outperforms SCAN baselines across nearly all problem settings and domain counts (\tabref{table:fieldguide2}); in sparser settings, the accuracy gap is 10-30\%. In this dataset, SCAN performs only slightly above chance, reflecting perhaps a total misalignment of learned class distinctions with true species boundaries. We do not believe that SCAN is too weak to effectively detect image groupings on this data; instead we acknowledge that the domain information available to DDFA~(SPI) (and not to SCAN) is informative for ensuring recovery of the true class distinction between species as opposed to the visually striking distinction between adult and immature life stages. Results from more domains are available in \appref{sec:appendix-experimental-results}.
\update{On FieldGuide-28 (\tabref{table:fieldguide28}), DDFA outperforms SCAN baseline in all evaluated sparsity levels, with the highest observed accuracy difference ranging above 30-40\%}.

%% file: sections/07_conclusion.tex
Our theoretical results demonstrate that under LLS, 
we can leverage shifts among previously seen domains
to recover labels in a purely unsupervised manner,
and our practical instantiation of the DDFA framework
demonstrates both (i) the practical efficacy of our approach;
and (ii) that generic unsupervised methods can play a key role
both in clustering discriminator outputs, 
and providing weights for initializing the discriminator.
We believe our work is just the first step
in a new direction for leveraging structural assumptions
together with distribution shift to perform unsupervised learning.

%% file: sections/08_limitations.tex
\update{\textbf{Assumptions {} {}} Our approach is limited by the set of assumptions needed (label shift, as many data domains as true latent classes, known true number of classes $k$, and other assumptions established in A.1-A.4). Future work should aim to relax these assumptions.}

\update{\textbf{Theory {} {}} 
The work does not include finite sample bounds for the DDFA algorithm. In addition, we do not have a formal guarantee that the clustering heuristic in the Discretize step of DDFA will retrieve pure anchor sub-domain clusters under potentially noisy black-box prediction of $q(d|x)$. This problem is complicated by the difficulty of reasoning about the noise that may be produced by a neural network or other complex non-linear model (acting as the black-box domain discriminator), and by the lack of concrete guarantees that K-means will recover the anchor subdomains among its recovered clusters. In particular, in the case in which the anchor subdomains do not contain all of the mass (equivalently, there are some $x$ which could belong to more than one $y$), the arbitrary distribution of mass outside of the anchor subdomains makes it difficult to reason about the behavior of K-means.}

\textbf{DDFA Framework~~} Within the LLS setup, several components of the DDFA framework
warrant further investigation: 
(i) the deep domain discriminator 
can be enhanced in myriad ways; 
(ii) for clustering discriminator outputs,
we might develop methods specially tailored to our setting \update{to replace the current generic clustering heuristic};
(iii) clustering might be replaced altogether
with geometrically informed methods 
that directly identify the corners of the polytope;
(iv) the theory of LLS can be extended beyond identification
to provide statistical results that might hold
when $q(d|\update{x})$ can only be noisily estimated,
and when only finite samples are available
for the induced topic modeling problem;
\update{(v) when the number of true classes $k$ is unknown, we may develop approaches to estimate this $k$.}

\update{\textbf{Semi-synthetic Experiments {} {}} Semi-synthetic experiments present an ideal environment for evaluating under the precise label shift assumption. Evaluating on datasets in which the separation into domains is organic, and the label shift is inherent is an interesting direction for future work.
}

%% file: sections/09_acknowledgements.tex
SG acknowledges Amazon Graduate Fellowship and JP Morgan PhD Fellowship for their support.
ZL acknowledges Amazon AI, Salesforce
Research, Facebook, UPMC, Abridge, the PwC Center, the Block Center, the Center for Machine
Learning and Health, and the CMU Software Engineering Institute (SEI) via Department of Defense
contract FA8702-15-D-0002, for their generous support of ACMI Lab’s research on machine learning
under distribution shift.

%% file: sections/appendix_A_proofs_lemmas.tex
In this section, we present several new lemmas which are required to prove Theorem \ref{thm:anchor-subdomain-continuous}, and provide proofs. We also provide proofs for Lemmas \ref{lemma:solve-domain-agnostic}, \ref{lemma:solve-domain-informed}, and \ref{lemma:dd_maps_to_unique}\update{.}

\begin{lemma}
\label{lemma:each_y_nonzero_weight}
Let distribution $Q$ over random variables $X,Y,D$ satisfy A.1--A.4. Then for all $y \in \out$, $q(y) > 0$. That is, all labels have nonzero probability under $Q$.
\end{lemma} 

\begin{proof}[Proof of Lemma \ref{lemma:each_y_nonzero_weight}]

Proof by contradiction. Let $y \in \out$ with $q(y) = 0$.
\begin{align*}
q(y) &= \sum\limits_{d \in \mathcal{R}} q(d)q(y|D = d) \\
&= \sum\limits_{d \in \mathcal{R}} \gamma_y q(y|D = d) \\
&= \sum\limits_{d \in \mathcal{R}} \frac{1}{r} q(y|D = d) \\
&= \frac{1}{r} \sum\limits_{d \in \mathcal{R}}q(y|D = d) \,.
\end{align*}
Since $q(y|D = d) \ge 0$ for all $d \in \mathcal{R}$, we see that if $q(y) = 0$, then $q(y|D = d) = 0 $ for all $ d \in \mathcal{R}$.

Then $[\bQ_{Y|D}]_{ y, d } = 0 $ for all $ d \in \mathcal{R}$. Then there is a row (row $y$) in the matrix $\bQ_{Y|D}$ in which every entry is 0, so $\bQ_{Y|D}$ cannot be full row rank $k$. This violates assumption A.2.
Then by contradiction we have shown $q(y) > 0$.
\end{proof}

\begin{lemma}
\label{lemma:onehot}
Let distribution $Q$ over random variables $X,Y,D$ satisfy Assumptions A.1--A.4.
Let $x \in \mathcal{X}$ such that $q(x) > 0$.
Then if $x \in A_y$ for some  $y \in \out$, we have that $q(y|X = x) = 1, \textrm{ and for all } y' \in \out \setminus \{ y \}$, $q(y'|X = x) = 0 $.
The converse is also true: if $ q(y|X = x) = 1 \textrm{ for some } y \in \out$ and $q(y'|X = x) = 0 \;  \forall y' \in \out \setminus \{ y \} $, then we know that $x \in A_y$.

\end{lemma}

\begin{proof}[Proof of Lemma \ref{lemma:onehot}]

We prove directions one at a time.

\paragraph{Forward direction.}

Assume $x \in A_y$.
\begin{align*} 
    q(x) &= \sum\limits_{y'' \in \out} q(y'')q(x|Y = y'') \\
     &= q(y)q(x|Y = y) + \sum\limits_{y' \in \out \setminus \{y\}}q(y')q(x|Y = y') \\
      &=  q(y)q(x|Y = y) + \sum\limits_{y' \in \out \setminus \{y\}} q(y') \left( 0 \right) \\
       &= q(y)q(x|Y = y).
\end{align*}

Recalling $q(x|y) > 0$ (by A.4) and $q(y) > 0$ (by Lemma \ref{lemma:each_y_nonzero_weight}), we know that $q(x) =  q(y)q(x|Y = y) > 0$. Then $q(y|X = x) = \cfrac{q(y)q(x|Y = y)}{q(x)} = \cfrac{q(x)}{q(x)} = 1$.
\update{Because probabilities sum to 1}, $ q(y|X = x) + \sum\limits_{y' \in \out \setminus \{y\}} q(y'|X = x) = 1$.
Then because $q(y|X = x) = 1$, we have :
$\sum\limits_{y' \in \out \setminus \{y\}} q(y'|X = x) = 0$. 
Then for all $y' \in \out \setminus \{y\}$, it must be that $q(y'|X = x) = 0$.
Then we have shown $q(y|X = x) = 1$, and for all $ y' \in \out \setminus \{y\}, \; q(y'|X = x) = 0$.

\paragraph{Converse.}

Assume $q(y|X = x) = 1$ and for all $y' \in \out \setminus \{y\}, q(y'|X = x) = 0$.
\update{We recall that} $q(x) > 0$.
Also, $q(y) > 0$ by Lemma \ref{lemma:each_y_nonzero_weight}.
Then $q(x | Y = y) = \cfrac{q(y|X = x)q(x)}{q(y)} = \cfrac{(1)q(x)}{q(y)} > 0$.
Let $y' \in \out \setminus \{y\}$. Then $q(x | Y = y') = \cfrac{q(y'|X = x)q(x)}{q(y')} = \cfrac{(0)q(x)}{q(y')} = 0$.
Then because $q(x | Y = y) > 0$ and $\forall y' \in \out \setminus \{y\}, \; q(x | Y = \update{y'})= 0$, we see that $x \in A_y$.
\end{proof}

\begin{lemma}
\label{lemma:P_d_y-linearly-independent}
Let random variables $X,Y,D$ and distribution $Q$ satisfy Assumptions A.1--A.4. Then, the matrix $\bQ_{D|Y}, \textrm{ defined as an } r \times k \textrm{ matrix whose elements \update{are} } [\bQ_{D|Y}]_{i,j} = Q(D=i|Y=j)$, and \update{in which} each column is a conditional distribution over the domains given a label, has linearly independent columns.
\update{Furthermore}, $\bQ_{D|Y}$ can be computed directly from only $\bQ_{Y|D}$. 
\end{lemma}

\begin{proof}[Proof of Lemma \ref{lemma:P_d_y-linearly-independent}]

Let random variables $X,Y,D$ and distribution $Q$ satisfy Assumptions A.1--A.4.

Each $[\bQ_{D|Y}]_{d,y} = q(d | Y = y) = \cfrac{q(y|D = d)q(d)}{q(y)} = \cfrac{q(y|D = d) \gamma_d }{q(y)} = \cfrac{q(y|D = d)}{r q(y)}$.

Since each $y$th column of $\bQ_{D|Y}$ is a probability distribution that sums to 1, and $rq(y)$ is constant down each \update{$y$th} column, we can obtain $\bQ_{D|Y}$ by simply taking $\bQ_{Y|D}^\top$, in which each $[\bQ_{Y|D}^\top]_{d,y} = [\bQ_{Y|D}]_{y,d} = q(y|D = d)$, and normalizing the columns so they sum to 1.

The matrix $\bQ_{Y|D}$ has linearly independent rows by Assumption A.2. Then $\bQ_{Y|D}^\top$ has linearly independent columns. Scaling these columns by a nonzero value does not change their linear independence, so the columns of $\bQ_{D|Y}$ are also linearly independent.

Then matrix $\bQ_{D|Y}$ has linearly independent columns, and can be computed by taking $\bQ_{Y|D}^\top$ and normalizing its columns.

\end{proof}

\begin{lemma}
\label{lemma:q_d_x_y}
Let random variables $X,Y,D$ and distribution $Q$ satisfy Assumptions A.1--A.4.
Let $d \in \mathcal{R}, x \in \mathcal{X}, y \in \mathcal{Y}$.
Then $q(d|X = x,Y = y) = q(d|Y = y)$.

\end{lemma}

\begin{proof}[Proof of Lemma \ref{lemma:q_d_x_y}]

\begin{align*}
q(d|X = x,Y = y) &= \cfrac{q(x | D = d, Y = y)\update{q(d | Y = y)}}{q(x | Y = y)} \\
&= \cfrac{p_d(x | Y = y)q(d | Y = y)}{q(x | Y = y)} \\
&= \cfrac{p(x | Y = y)q(d | Y = y)}{q(x | Y = y)} \\
&= \cfrac{q(x | Y = y)q(d | Y = y)}{q(x | Y = y)} \\
&= q(d | Y = y).
\end{align*}
\end{proof}

\newtheorem*{lemma:solve-domain-agnostic-restated}{Lemma \ref{lemma:solve-domain-agnostic}}
\begin{lemma:solve-domain-agnostic-restated}
\textit{We restate this lemma, first presented in \secref{sec:theory}, for convenience.} Let the distribution $Q$ over random variables
$X,Y,D$ satisfy Assumptions A.1--A.4. Then the matrix $\bQ_{Y|D}$ and $f(x) = q(d|x)$ 
uniquely determine $q(y|x)$ for all $y \in \out$ 
and $x \in \mathcal{X}$ such that $q(x) > 0$.
\end{lemma:solve-domain-agnostic-restated}

\begin{proof}[Proof of Lemma \ref{lemma:solve-domain-agnostic}]

Let distribution $Q$ over random variables $X,Y,D$ satisfy Assumptions A.1-A.4.
Let $x \in \mathcal{X}$ with $q(x) > 0$, and $y \in \out$.
Assume we know $\bQ_{ Y | D }$ and $[f(x)]_d = q(d|X =x)$.
Notice that, for all $x \in \mathcal{X}, d \in \mathcal{R}$,
\begin{align*}
q(d|X = x) &= \update{\sum\limits_{y' \in \out}  q(d|X = x,Y = y') q(y' | X = x)}.
\end{align*}
Now using Lemma \ref{lemma:q_d_x_y},
\begin{align*}
q(d|X = x) &= \update{\sum\limits_{y' \in \out}  q(d | Y = y') q(y' | X = x)}.
\end{align*}

Define the vector-valued function $g : \mathcal{X} \to \mathbb{R}^k$ such that $[g(x)]_y = q(y|X = x)$ for all $x \in \textrm{supp}_{\update{Q}}(X)$.
$\bQ_{D|Y}$ is a matrix of shape $r \times k$, with $[\bQ_{D|Y}]_{i,j} = Q(D=i|Y=j)$. It can be computed from $\bQ_{Y|D}$ and has linearly independent columns---both facts shown in Lemma \ref{lemma:P_d_y-linearly-independent}.

Then $[f(x)]_d = q(d|X = x) = \bQ_{D|Y}[d, \boldsymbol :] g(x)$, a product between the $d$th row vector of $\bQ_{D|Y}$ and the column vector $g(x)$.
Then $f(x) = \bQ_{D|Y} g(x)$.

This system is a linear system with $r \ge k$ equations. Recalling that $\bQ_{D|Y}$ has $k$ linearly independent columns, we can select any $k$ linearly independent rows of $\bQ_{D|Y}$ to solve the equation for the true, unique solution for the unknown vector $g(x)$.
Another way to describe this is with the pseudo-inverse: $g(x) = (\bQ_{D|Y})^\dagger f(x)$.
Then we have $[g(x)]_y = q(y|X = x)$ for all $y \in \mathcal{Y}$.

\end{proof}

\newtheorem*{lemma:solve-domain-informed-restated}{Lemma \ref{lemma:solve-domain-informed}}
\begin{lemma:solve-domain-informed-restated}
\textit{We restate this lemma, first presented in \secref{sec:theory}, for convenience.} Let the distribution $Q$ over random variables
$X,Y,D$ satisfy Assumptions A.1--A.4. Then for all $y \in \out$, and $x \in \mathcal{X}$ such that $q(x,d) > 0$.
the matrix $\bQ_{Y|D}\,$ and $q(y|x)$ 
uniquely determine $q(y|x,d)$.
\end{lemma:solve-domain-informed-restated}

\begin{proof}[Proof of Lemma \ref{lemma:solve-domain-informed}]

Let distribution $Q$ over random variables $X,Y,D$ satisfy Assumptions A.1-A.4.
Let $x \in \mathcal{X}, d \in \mathcal{R}$ with $q(x,d) > 0$, and $y \in \out$.

Assume we know matrix $\bQ_{ Y | D }$ and $q(y'|X = x), \; \forall y' \in \out$.
We can compute $\bQ_{ D | Y }$ from $\bQ_{ Y | D }$ via Lemma \ref{lemma:P_d_y-linearly-independent}.

\begin{align*}
q(y | X = x, D = d) &= \cfrac{q(y,x,d)}{q(x, d)} \\
&= \cfrac{q(d | X = x, Y = y)q(y | X = x)q(x)}{q(d|X = x)q(x)} \,.
\end{align*}

Using Lemma \ref{lemma:q_d_x_y}, $q(d | X = x, Y = y) = q(d | Y = y)$. \update{We apply this property.}
\begin{align*}q(y | X = x, D = d) &= \cfrac{q(d | Y = y)q(y | X = x)q(x)}{q(d|X = x)q(x)} \\ &= \cfrac{q(d | Y = y)q(y | X = x)}{q(d|X = x)}.
\end{align*}

The denominator $q( d |X = x)$ is constant across all values of $y$, so we can write that  $q(y | X = x, D = d) \; \propto \; q(d | Y = y)q(y | X = x)$ and normalize to find the probability:

\begin{align*}
q(y | X = x, D = d) = \cfrac{q(d | Y = y)q(y | X = x)}{\sum\limits_{y' \in \out} q(d | Y = y')q(y' | X = x) }.
\end{align*}

We know $q(d | Y = y)$ as $[\bQ_{ D | Y }]_{d,y}$, and every $q(d | Y = y')$, where $y' \in \out \update{\setminus \{y\}}$, as $[\bQ_{ D | Y }]_{d,y'}$.
We also know $q(y | X = x)$ and every $q(y' | X = x)$ where $y' \in \out \update{\setminus \{y\}}$, by the precondition assumptions.
Then we can compute $q(y | X = x, D = d)$.
\end{proof}

\newtheorem*{lemma:dd_maps_to_unique-restated}{Lemma \ref{lemma:dd_maps_to_unique}}
\begin{lemma:dd_maps_to_unique-restated}
\textit{We restate this lemma, first presented in \secref{sec:theory}, for convenience.} Let the distribution $Q$ over random variables
$X,Y,D$ satisfy Assumptions A.1--A.4. Then for all $x,x^\prime$ in anchor sub-domain, i.e., $ x, x^\prime \in A_y$ 
for a given label $y\in \out$, we have $f(x) = f(x^\prime)$.
Further, for any $y \in \out$, if $x \in A_y , x^\prime \notin A_y$, 
then $f(x) \neq f(x^\prime)$. 
\end{lemma:dd_maps_to_unique-restated}

\begin{proof}[Proof of Lemma \ref{lemma:dd_maps_to_unique}]
Let distribution $Q$ over random variables $X,Y,D$ satisfy Assumptions A.1-A.4.
Recall $f: \mathbb{R}^p \to \mathbb{R}^r$ is a vector-valued oracle function such that $[f(x)]_{d} = q(d|X = x)$ for all $x \in \textrm{supp}_Q(X)$.
Also let us recall that $\bQ_{D|Y} \textrm{ is defined as an } r \times k \textrm{ matrix whose elements } [\bQ_{D|Y}]_{i,j} = Q(D=i|Y=j), \textrm{ and each column is a conditional distribution over the domains given a label}$. It has linearly independent columns by Lemma \ref{lemma:P_d_y-linearly-independent}.
\smallskip

First recognize that \update{for all} $d \in \mathcal{R}, x \in \mathcal{X}$ \update{such that} $q(x) > 0$,

\begin{align*}[f(x)]_d = q(d|X = x) &= \sum\limits_{y'' \in \out} q(d, y''|X = x) . \\
&= \sum\limits_{y'' \in \out} q(d| Y = y'', X = x) q(y'' | X = x).
\end{align*}

Using Lemma \ref{lemma:q_d_x_y}, $q(d | X = x, Y = y) = q(d | Y = y)$. \update{We apply this property.}

\begin{align*}
[f(x)]_d = q(d|X = x) &= \sum\limits_{y'' \in \out}  q( d | Y = y'') q(y'' | X = x).
\end{align*}

\update{Then we can} write $f(x) = \sum\limits_{y'' \in \out}   q(y''|X = x) \bQ_{D|Y}[\boldsymbol :, y'']$, where $\bQ_{D|Y}[\boldsymbol :, y'']$ is the $y''$th column of $\bQ_{D|Y}$.
Now we could also rewrite $f(x) = \bQ_{D|Y} \left[ Q(Y = 1 | X = x) \; ... \; Q(Y = k | X = x)\right]^\top$.

We now prove two key components of the lemma.
Let $y \in \out$.
Let $x \in A_y$ \update{such that} $q(x) > 0$.

\paragraph{Points in same anchor sub-domain map together.}
    
    Let $x' \in A_y$ such that $q(x') > 0$. We now seek to show that $f(x) = f(x')$.
Recall that $x, x' \in A_y$. By Lemma \ref{lemma:onehot}, $q(y | X = x) = q(y | X = x') = 1$. Also by lemma \ref{lemma:onehot}, $\forall y'' \in \out \setminus \{ y \}, q(y'' | X = x) = q(y'' | X = x') = 0$.
Then for all $y'' \in \out$, $q(y'' | X = x) = q(y'' | X = x')$.

Therefore, $\forall d \in \mathcal{R}$,

\begin{align*}[f(x)]_d = q(d | X = x) &= \sum\limits_{y'' \in \out}  q(d | Y = y'') q(y'' | X = x) \\ &= \sum\limits_{y'' \in \out}  q(d|Y = y'') q(y'' | X = x') \\ &= q(d | X = x') = [f(x')]_d.
\end{align*}
Then $f(x) = f(x')$.

\paragraph{Point outside of the anchor sub-domain does not map with points in the anchor sub-domain}.
Let $x_0 \notin A_y$ such that $q(x_0) > 0$. We now seek to show that $f(x) \neq f(x_0)$. Because $x_0 \notin A_y$ with $q(x_0) > 0$, and because $A_y$ \update{contains all $x$ such that $q(x) > 0$, $q(y|X = x)= 1$, and $q(y'|X = x) = 0$ for all $y' \in \out \setminus \{y\}$,} then by Lemma \ref{lemma:onehot}, it must be that one of the following cases is true:

\begin{itemize}
    \item \textbf{Case 1: }$q(y|X = x_0) \neq 1$
    
    \item \textbf{Case 2: }$q(y'|X = x_0) > 0$ for some $y' \in \out \setminus \{y\}$.
\end{itemize}

In all circumstances, \update{there exists some} $y'' \in \out : q(y''|x_0) \neq \update{q(y''|x)}$. Then,

\begin{align*}
\left[ Q(Y = 1 | X = x)  ...  Q(Y = k | X = x)\right]^\top \neq \left[ Q(Y = 1 | X = x_0)  ...  Q(Y = k | X = x_0)\right]^\top.
\end{align*}

Because $\bQ_{D|Y}$ has linearly independent columns (shown in Lemma \ref{lemma:P_d_y-linearly-independent}), we now know that 
\begin{align*}
    f(x) &= \bQ_{D|Y} \left[ Q(Y = 1 | X = x) \; ... \; Q(Y = k | X = x)\right]^\top \\
& \neq \bQ_{D|Y} \left[ Q(Y = 1 | X = x_0) \; ...\;  Q(Y = k | X = x_0)\right]^\top = f(x_0) \,. \\
\end{align*}

So $f(x) \neq f(x_0)$.

\end{proof}

%% file: sections/appendix_B_proof_theorem.tex
\newtheorem*{thm:anchor-subdomain-continuous-restated}{Theorem \ref{thm:anchor-subdomain-continuous}}
\begin{thm:anchor-subdomain-continuous-restated}
\textit{We restate this theorem, first presented in \secref{sec:theory}, for convenience.} Let the distribution $Q$ over random variables
$X,Y,D$ satisfy Assumptions A.1--A.4. Assuming access to the joint distribution $q(x, d)$, \update{and knowledge of the number of true classes $k$}, 
we show that the following quantities are identifiable: 
(i) $\bQ_{Y|D}$, (ii) $q(y|X = x)\,$,
for all  $x \in \inpt$ that lies 
in the support (i.e. $q(x) > 0$); 
and (iii) $q(y|X = x,D = d)\,$, 
for all $x \in \inpt$ and $d \in \calR$
such that  $q(x,d) > 0$. 
\end{thm:anchor-subdomain-continuous-restated}

\begin{proof}[Proof of Theorem \ref{thm:anchor-subdomain-continuous}]

Let distribution $Q$ over random variables $X,Y,D$ satisfy Assumptions A.1-A.4.

Recall $f: \mathcal{X} \to \mathbb{R}^r$ is a vector-valued oracle function such that $[f(x)]_{d} = q(d|X = x)$ for all $x \in \textrm{supp}_Q(X)$. It is known because we know the marginal $q(x, d)$. 
Let $y \in \out$.
Then by Lemma \ref{lemma:dd_maps_to_unique}, \update{$f$} sends every $x \in A_y$ (and no other $x \notin A_y$) to the same value. We overload notation to denote this as $f(A_y)$.
Then $Q(f(X) = f(A_y)) = Q(X \in A_y) \ge \epsilon$.
Then in the marginal distribution of $f(X)$ with respect to distribution $Q$, there is a distinct point mass on each $f(A_y)$, with mass at least $\epsilon$.

Because we know the marginal $q(x,d)$, we know the marginal $q(x)$, so we can obtain the distribution of $f(X)$ with respect to distribution $Q$.
If we analyze the marginal distribution of $f(X)$ with respect to distribution $Q$, and recover all point masses with mass at least $\epsilon$, we can recover no more than $\mathcal{O}\left(\nicefrac{1}{\epsilon}\right)$ such points. We set $m \in \mathbb{Z}^+$ so that the number of points we recovered is $m - 1$.

We denote a mapping $\psi : \mathbb{R}^r \to [m]$. This mapping sends each value of $f(x)$ corresponding to a point mass with mass at least $\epsilon$ to a unique index in $\{1, ..., m - 1\}$. It sends any other value in $\mathbb{R}^p$ to $m$. We note that the ordering of the point masses might have $(m-1)!$ permutations.

Notice that the point mass on each $f(A_y)$ must be recovered among these $m - 1$ masses. 
Recall that \update{for all} $y \in \out$, $ f(x) = f(A_y)$ \update{if and only if } $x \in A_y$.
Then for all $y \in \out$, $\psi(f(\update{x})) = \psi(f(A_y))$ \update{if and only if } $\update{x} \in A_y$, because $\psi$ does not send any other value in $\mathbb{R}^r$ besides $f(A_y)$ to $\psi(f(A_y))$.

\smallskip

For convenience, we now define a mapping $c: \mathcal{X} \to [m]$ such that $c = \psi \circ f$.
We will also abuse notation here to denote $c(A_y) = \psi(f(A_y))$.
Then $c(X)$ is a discrete, finite random variable that takes values in $[m]$. As $c$ is a pushforward function on $X$, $c(X)$ satisfies the label shift assumption because $X$ does \update{(i.e., when conditioning on $Y$, the distribution of $c(X)$ is domain-invariant)}.

We might now define a matrix $\bQ_{c(X)|D}$ in which each entry $[\bQ_{c(X)|D}]_{ i, d } = Q(c(X) = i | D = d)$.
\update{We recall that we know the number of true classes $k$. Then we know that there is a (possibly unique) unknown decomposition of the following form:}
\begin{align*}
\update{q(c(X) | d)} &= \update{\sum\limits_{y \in \out} q(c(X) | Y = y, D = d)q(y|D = d)}
\end{align*}

Using the label shift property,
\begin{align*}
\update{q(c(X) | d)} &= \update{\sum\limits_{y \in \out} q(c(X) | Y = y)q(y|D = d)}.
\end{align*}
\update{To express this decomposition in matrix form, we write $\bQ_{c(X)|D} = \bQ_{c(X)|Y} \bQ_{ Y | D }$. Now we make observations about the unknown $\bQ_{c(X)|Y}$.} For all $ y \in \out$,
\begin{align*}
     Q(c(x) = c(A_y) | Y = y) &= Q(X \in A_y | Y = y) > 0 \,. \\ 
Q(c(x) = c(A_y) | Y \neq y) &= Q(X  \in  A_y | Y \neq y) = 0 \,.
\end{align*}

Then for each $y \in \out$, the row of $\bQ_{c(X)|Y}$ with row index $c(A_y)$ is positive in the $y$th column, and zero everywhere else.
Restated, for each $y \in \out$, there is some row with positive entry exactly in $y$th column. This is precisely the anchor word assumption for a discrete, finite random variable.
\update{We already know that $\bQ_{ Y | D }$ is full row-rank, so because $\bQ_{c(X)|Y}$ satisfies the anchor word assumption, we can identify $\bQ_{ Y | D }$ up to permutation of rows by Theorem \ref{thm:separability-discrete}. In other words, when we set the constraint that the recovered $\bQ_{c(X)|Y}$ must have $k$ columns and satisfy anchor word and the recovered $\bQ_{ Y | D }$ must have $k$ rows and be full row-rank, any solution to the decomposition $\bQ_{c(X)|D} = \bQ_{c(X)|Y} \bQ_{ Y | D }$ must identify the ground truth $\bQ_{ Y | D }$, up to permutation of its rows.}

\smallskip

\end{proof}

%% file: sections/appendix_C_updated.tex
In this section we reason about our choice of cross entropy loss to estimate the domain discriminator. We show that in population, when optimizing over a sufficiently powerful class of functions, the minimizer of the cross entropy loss is the conditional distribution over domains given input $X$. 

Define the vector-valued function $z : \mathcal{R} \to \mathbb{R}^r$ such that  $z(D)$ is a one-hot random vector of length $r$, such that $[z(D)]_i = 1, \textrm{ iff } D = i$.
Then we write the cross-entropy objective in expectation over the input random variable $X$ and target random variable $D$ as: 

\begin{align*}
    \cL &= \expect_{(X,D) \sim Q}\left[ - \sum_{i=1}^{i=r} [z(D)]_i\log([f(X)]_i)\right] \\
        &= \expect_{X}\expect_{D|X}\left[ - \sum_{i=1}^{i=r} [z(D)]_i\log([f(X)]_i)\right]
\end{align*}

where the second line follows by splitting the expectations using the law of iterative expectations. For ease of notation we denote conditioning on $X = x$ as just conditioning on $X$.
We now move the inner expectation into
the summation over the domain random variable and pull the terms that do not depend on $D$ outside this inner expectation: 

\begin{align*}
    \cL &= \expect_{X}\left[- \sum_{i=1}^{i=r} \expect_{D|X} \left[[z(D)]_i\log([f(X)]_i)\right]\right] \\
        &= \expect_{X}\left[- \sum_{i=1}^{i=r} \log([f(X)]_i) \expect_{D|X} [[z(D)]_i]\right] 
\end{align*}

Since $[Z(D)]_i$ is a random variable which is 1 when $D=i$ conditioned on $X$, and 0 otherwise, the inner expectation can be simplified as follows: 
$\expect_{D|X} [[z(D)]_i] = 1 \times Q(D=i|X) + 0 \times (1-Q(D=i|X))$, giving 

\begin{align*}
    \cL  &= \expect_{X}\left[- \sum_{i=1}^{i=r} \log([f(X)]_i)Q(D=i|X)\right]
\end{align*}

In order to learn a distribution, we constrain our search space to the subset of functions from $\update{\mathbb{R}}^p \to [0,1]^r$ that project to a simplex $\Delta^{r-1}$. We now look to find the minimizer of the above equation from within this subset for each value of $X$. This corresponds to minimizing each term that contributes to the outer expectation over $X$, leading to the overall minimizer of the cross-entropy loss. 

We formulate this objective along with the Lagrange constraint modeling the sum of components of $f(X)$ adding to 1. Note that we do not add the constraint of each component of $f(X)$ lying between 0 and 1, but it is easy to see that the resulting solution satisfies this constraint.  

$$J = \min_{[f(X)]_1 ... [f(X)]_r}- \sum_{i=1}^{i=r} \log([f(X)]_i)Q(D=i|X) + \lambda \left(\sum_{i=1}^{i=r}[f(X)]_i - 1 \right) $$

Setting partial derivative with respect to $[f(X)]_i$ to $0$, we get  $-\cfrac{Q(D=i|X)}{[f^{\star}(X)]_i} + \lambda = 0$
and $[f^{\star}(X)]_i = \frac{1}{\lambda} Q(D=i|X)$, where we use $f^{\star}$ to denote the minimizer. 

The last piece is to solve for the Lagrange multiplier $\lambda$. This can be achieved by applying Karush-Kuhn-Tucker (KKT) conditions which suppose that the optimal solution lies on the constraint surface. This gives, $\sum_{i=1}^{i=r} [f^{\star}(X)]_i = 1$ which implies $\sum_{i=1}^{i=r} \frac{1}{\lambda} Q(D=i|X) = 1$. Since sum over domains $Q(D|X) = 1$, we get $ \lambda = 1$.

Plugging this in, we finally get $[f^{\star}(X)]_i = Q(D=i|X)$. Thus, the optimal $f^*$ obtained by minimizing the cross entropy objective will in fact recover the oracle domain discriminator. As mentioned, it is clear that the resulting solution satisfies the required conditions of valid probability distributions. 

%% file: sections/appendix_D_experiment_details.tex
Our code is available at
\url{https://github.com/acmi-lab/Latent-Label-Shift-DDFA}.
Here we present the full generation procedure for semisynthetic example problems, and discuss the parameters.

\begin{enumerate}
    \item Choose a Dirichlet concentration parameter $\alpha > 0$, maximum condition number $\kappa \ge 1$ (with respect to 2-norm), and domain count $r \ge k$.
    \item For each $y \in [k]$, sample $p_d(y) \sim \textrm{Dir}(\frac{\alpha}{k} \mathbf{1}_k)$.
    \item Populate the matrix $\bQ_{Y|D}$ with the computed $p_d(y)$s. If $\textrm{cond}(\update{\bQ_{Y|D}}) > k$, return to step 2 and re-sample.
    \item Distribute examples across domains according to $\bQ_{Y|D}$, for each of train, test, and valid sets. This procedure entails creating a quota number of examples for each (class, domain) pair, and drawing datapoints without replacement to fill each quota. We must discard excess examples from some classes in the dataset due to class imbalance in the \update{$\bQ_{Y|D}$} matrix. Due to integral rounding, domains may be \textit{slightly} imbalanced.
    \item Conceal true class information and return $(x_i,d_i)$ pairs.
\end{enumerate}

It is important to note the role of $\kappa$ and $\alpha$ in the above formulation. Although they are unknown parameters to the classification algorithm, they affect the sparsity of the $\bQ_{Y|D}$ and difficulty of the problem. Small $\alpha$ encourages high sparsity in $p_d(y)$, and large $\alpha$ causes $p_d(y)$ to tend towards a uniform distribution. \update{We observe an example of the effects of $\alpha$ in \figref{fig:matrices_sampled}.} $\kappa$ has a strong effect on the difficulty of the problem. Consider the case when $k=2$. When $\kappa = 1$, the only potential $\bQ_{Y|D}$ matrices are $\mathbf{I}_2$ up to row permutation (which means that domains and classes are exactly correlated, so the domain indicates the class and the problem is supervised). In the other limit, if $\kappa \to +\infty$, we may generate $\bQ_{Y|D}$ matrices that are nearly singular, breaking needed assumptions for domain discriminator output to uniquely identify true class of anchor subdomains. $\kappa$ also helps control the class imbalance (if a row of $\bQ_{Y|D}$ is small, indicating that the class is heavily under-represented across all domains, the condition number will increase).

\begin{figure}[t]

  \centering
\subfigure[$\alpha : 0.5, \kappa : 3$]{ 
    {} {}$\begin{bmatrix}
            0.17 & 0.65 \\
            0.83 & 0.35
        \end{bmatrix}${} {}}

\subfigure[$\alpha : 3, \kappa : 5$]{ 
    {} {}$\begin{bmatrix}
            0.37 & 0.06 \\
            0.63 & 0.94 
        \end{bmatrix}${} {}}
\quad
\subfigure[$\alpha : 10, \kappa : 7$]{ 
    {} {}$\begin{bmatrix}
            0.42 & 0.25 \\
            0.58 & 0.75
        \end{bmatrix}${} {}}
  \caption{Example $\bQ_{Y|D}$ matrices sampled for FieldGuide-2 with 2 classes and 2 domains. Each column represents the distribution across classes $\update{p_d(y)}$ for a given domain. At small $\alpha$, each $\update{p_d(y)}$ is likelier to be ``sparse'' (\update{our definition is an informal one meaning not that there are many zero entries, but instead that the distribution is heavily concentrated in a few classes}). \update{At} large $\alpha$, $\update{p_d(y)}$ tends toward a uniform distribution in which classes are represented evenly.}
   \label{fig:matrices_sampled}
\end{figure}

\subsection{FieldGuide-2 and FieldGuide-28 Datasets}

The dataset and description is available at \url{https://sites.google.com/view/fgvc6/competitions/butterflies-moths-2019}. From this data we create two datasets FieldGuide-2 and FieldGuide-28. For FieldGuide-28 we select the 28 classes which have 1000 datapoints in the training file. Since the test set provided in the website does not have annotations, we manually create a test set by sampling 200 datapoints from training file of each of the 28 classes. Therefore, we finally have 22400 \update{training} points and 5600 testing points. The FieldGuide-2 dataset is created by considering two classes from the created FieldGuide-28 dataset. 
 
 \subsection{Hyperparameters and Implementation Details: SCAN baseline}

In all cases, we initialize the SCAN \citep{SCAN} network with the clustering head attached, sample data according to the $\bQ_{Y|D}$ matrix, and predict classes.
With the Hungarian algorithm, implemented in \citep{crouse2016implementing, 2020SciPy-NMeth}, we compute the highest true accuracy among any permutation of these labels (denoted ``Test acc'').

\begin{itemize}
    \item \textbf{CIFAR-10 and CIFAR-20 Datasets \citep{CIFAR} }
    We use  ResNet-18 \citep{he2016deep} backbone with weights trained by SCAN-loss and obtained from the SCAN repo  \url{https://github.com/wvangansbeke/Unsupervised-Classification}.
    We use the same transforms present in the repo for test data.
    
    \item \textbf{ImageNet-50 Dataset \citep{imagenet}}
    We use ResNet-50 backbone with weights trained by SCAN-loss and obtained from the SCAN repo.
    We use the same transforms present in the repo for test data.
    
    \item \textbf{FieldGuide-2 and FieldGuide-28 Datasets}
    For each of the two datasets, we pretrain a different SCAN baseline network (including pretext and SCAN-loss steps) on all available data from the dataset. The backbone for each is ResNet-18.
    For training the pretext task, we use the same transform strategy used in the repo for CIFAR-10 data (with mean and std values as computed on the FieldGuide-28 dataset, and crop size 224). For training SCAN, we resize the smallest image dimension to 256, perform a random horizontal flip and random crop to size 224. We also normalize. For validation we resize smallest image dimension to 256, center crop to 224, and normalize. 

    Hyperparameters
    used in training SCAN representations on both instances of the FieldGuide dataset were chosen as follows: starting with the recommended choices of hyperparameters for ImageNet (as was present in the SCAN repo), we made minimal changes to these only to avoid model degeneracy (training loss collapse).
\end{itemize}

\subsection{Hyperparameters and Implementation Details: DDFA (RI)}

This is the DDFA procedure with random initialization of the domain discriminator. The bulk of this procedure is described in Section \ref{sec:exp}, but for completeness we reiterate here.

We train \update{ResNet-50} \citep{he2016deep} (with random initialization and added dropout) based on the implementation from \url{https://github.com/kuangliu/pytorch-cifar} on images $x_i$ with domain indices $d_i$ as the label, choose best iteration by valid loss, pass all training and validation data through $\hat{f}$, and cluster pushforward predictions $\smash[]{\hat{f}(x_i)}$ into $m \ge k$ clusters with Faiss K-Means \citep{johnson2019billion}. We compute the $\smash[]{\hat{\bQ}_{c(X)|D}}$ matrix and run NMF to obtain $\smash[]{\hat{\bQ}_{c(X)|Y}}$, $\smash[]{\hat{\bQ}_{Y|D}}$. To make columns sum to 1, we normalize columns of $\smash[]{\hat{\bQ}_{c(X)|Y}}$, multiply each column's normalization coefficient over the corresponding row of $\smash[]{\hat{\bQ}_{Y|D}}$ (to preserve correctness of the decomposition), and then normalize columns of $\smash[]{\hat{\bQ}_{Y|D}}$.

Some NMF algorithms only output solutions satisfying the anchor word property \citep{arora-nmf-provably,kumar2013fast, SPA}. We found the strict requirement of an exact anchor word solution to lead to low noise tolerance. We therefore use the Sklearn implementation of standard NMF \citep{SklearnNMFImpl-01, tan2012automatic, scikit-learn}.

We predict class labels with Algorithm $\ref{alg:2}$. With the Hungarian algorithm, implemented in \citep{crouse2016implementing, 2020SciPy-NMeth}, we compute the highest true accuracy among any permutation of these labels (denoted ``Test acc''). With the same permutation, we reorder rows of \update{$\smash[]{\hat{Q}_{Y|D}}$}, then compute the average absolute difference between corresponding entries of $\smash[]{\hat{\bQ}_{Y|D}}$ and $\bQ_{Y|D} $ (denoted ``$\bQ_{Y|D}$ err'').

In order to make hyperparameter choices for final experiments, such as the choice of the NMF solver, clustering algorithm, and learning rate, we primarily consulted CIFAR-10 and CINIC-10 (similar to an extension of CIFAR-10) \citep{cinic10} final test task accuracy, and validation loss on other datasets (likely leading to an overfitting of our hyperparameter choices to CIFAR-10 and associated tasks). We applied the intuitions developed on these datasets when choosing hyperparameters for other datasets, instead of performing a test accuracy-driven sweep for each other dataset. Final runs used the following fixed hyperparameters:

\textbf{Common Hyperparameters {} {}}

\begin{itemize}

\item Hardware: A single NVIDIA RTX A6000 GPU was used for each experiment (on this hardware, trial length varies < 1 hour to ~24 hours, depending on the dataset).

\item Architecture: ResNet-50 with added dropout

\item Faiss KMeans number of iterations (niter): 100

\item Faiss Kmeans number of clustering redos (nredo): 5

\item Learning Rate: 0.001

\item Learning Rate Decay: Exponential, parameter 0.97

\item SKlearn NMF initialization: random

\end{itemize}

\textbf{Dataset-Specific Hyperparameters {} {}}

\begin{itemize}
    \item \textbf{CIFAR-10 Dataset}
    Training Epochs: 100. 
    Number of Clusters ($m$): 30
    
    \item \textbf{CIFAR-20 Dataset}
    Training Epochs: 100.
    Number of Clusters ($m$): 60
    
    \item \textbf{ImageNet-50 Dataset}
    DDFA (RI) was not evaluated for this dataset due to poor performance of the domain discriminator without an appropriate pre-seed in early trials.
    
    \item \textbf{FieldGuide-2 Dataset}
    \update{Training Epochs: 100}.
    \update{Number of Clusters ($m$): 10}
    
    \item \textbf{FieldGuide-28 Dataset}
    DDFA (RI) was not evaluated for this dataset due to poor performance of the domain discriminator without an appropriate pre-seed in early trials.
\end{itemize}

\subsection{Hyperparameters and Implementation Details: DDFA (SI) and DDFA (SPI)}

This is the DDFA procedure with SCAN initialization of the domain discriminator. \update{DDFA (SI) uses the SCAN pretext + SCAN loss pretraining steps, while DDFA (SPI) uses only the SCAN pretext step.}

The procedure is identical to the standard DDFA procedure, except that SCAN \citep{SCAN} pre-trained weights or SCAN \citep{SCAN} contrastive pre-text weights are used to initialize the domain discriminator before it is fine-tuned on the domain discrimination task. Hyperparameters used also differ.
When SCAN pretrained weights are available, we use those. When they are not, we train SCAN ourselves.

Like SCAN (RI), we used CIFAR-10 
and CINIC-10 final test accuracy to choose hyperparameters and make algorithm decisions. For other datasets, we consulted only validation domain discrimination loss. One exception to this rule was that preliminary low final DDFA (SI) performance on FieldGuide suggested that we should focus on instead evaluating DDFA (SPI) to avoid allowing the SCAN failure mode to negatively affect the domain discriminator pretrain representation. Final evaluation runs used the following fixed hyperparameters:

\textbf{Common Hyperparameters {} {}}

\begin{itemize}

\item Hardware: A single NVIDIA RTX A6000 GPU was used for each experiment (on this hardware, trial length varies < 1 hour to ~24 hours, depending on the dataset).

\item
Faiss KMeans number of iterations (niter): 100

\item
Faiss Kmeans number of clustering redos (nredo): 5

\item
Learning Rate: 0.00001

\item
Learning Rate Decay: Exponential, parameter 0.97

\item
SKlearn NMF initialization: random

\end{itemize}

\textbf{Dataset-Specific Hyperparameters {} {}}

\begin{itemize}
    \item \textbf{CIFAR-10 Dataset}
    
    Architecture: ResNet-18
    
    Pre-seed: Weights trained with SCAN pretext and SCAN-loss on entirety of CIFAR-10 (from SCAN repo).
    
    Training Epochs: 25
    
    Number of Clusters ($m$): 10
    
    Transforms used: Same as SCAN repo.
    
    \item \textbf{CIFAR-20 Dataset}
    
    Architecture: ResNet-18
    
    Pre-seed: Weights trained with SCAN pretext and SCAN-loss on entirety of CIFAR-20 (from SCAN repo).
    
    Training Epochs: 25
    
    Number of Clusters ($m$): 20
    
    Transforms used: Same as SCAN repo.
    
    \item \textbf{ImageNet-50 Dataset}
    
    Architecture: ResNet-50
    
    Pre-seed: Weights trained with SCAN pretext and SCAN-loss on entirety of ImageNet-50 (from SCAN repo).
    
    Training Epochs: 25
    
    Number of Clusters ($m$): 50
    
    Transforms used: Same as SCAN repo.
    
    \item \textbf{FieldGuide-2 Dataset}
    
    Architecture: ResNet-18
    
    Pre-seed: Weights trained with SCAN pretext on entirety of FieldGuide-2 (trained by us).
    
    Training Epochs: 30
    
    Number of Clusters ($m$): 2
    
    Transforms used for pretext: Same strategy as CIFAR-10 in SCAN repo with appropriate mean, std, and crop size 224.
    
    Transform used for SCAN: Resize to 256, Random horizontal flip, Random crop to 224, normalize
    
    Learning rate used for SCAN: 0.001 (other hyperparameters were same as in SCAN repo for CIFAR-10) 
    
    \item \textbf{FieldGuide-28 Dataset}
    
    Architecture: ResNet-18
    
    Pre-seed: Weights trained with SCAN pretext on entirety of FieldGuide-28 (trained by us).
    
    Training Epochs: 60
    
    Number of Clusters ($m$): 28
    
    Transforms used for pretext: Same strategy as CIFAR-10 in SCAN repo with appropriate mean, std, and crop size 224.
    
    Transform used for SCAN: Resize to 256, Random horizontal flip, Random crop to 224, normalize
    
    Learning rate used for SCAN: 0.01 (other hyperparameters were same as in SCAN repo for CIFAR-10)

\end{itemize}

%% file: sections/appendix_E_experiment_results.tex
Here we present additional experimental results. We also investigated evaluations on the Waterbirds dataset, but although DDFA showed some reasonable results, we did not find SCAN hyperparameters that lead to a successful SCAN baseline. Accordingly, we do not include these results as they do not present a reasonable comparison.

\begin{table}[ht]
        \caption{\emph{Results on CIFAR-10}. Each entry is produced with the averaged result of 5 different random seeds, formatted as mean $\pm$ standard deviation. With DDFA (RI) we refer to DDFA with randomly initialized backbone.  With DDFA (SI) we refer to DDFA's backbone initialized with SCAN. Note that in DDFA (SI), we do not leverage SCAN for clustering. $\alpha$ is the Dirichlet parameter used for generating label marginals in each domain, $\kappa$ is the maximum allowed condition number of the generated $\bQ_{Y|D}$ matrix, $r$ is number of domains. \update{``Test acc'' is classification accuracy, under the best permutation of the recovered classes, and ``$\bQ_{Y|D}$ err'' is the average entry-wise absolute error in the recovered $\bQ_{Y|D}$.}}
    \centering
    \scalebox{0.78}{
    \begin{tabular}{llllllll}
        \toprule
        \multirow{2}{*}{r} & \multirow{2}{*}{Approaches}    &\multicolumn{2}{c}{$\alpha: 0.5, \; \kappa: 4$} & \multicolumn{2}{c}{$\alpha: 3, \; \kappa: 4$} & \multicolumn{2}{c}{$\alpha: 10, \; \kappa: 8$} \\
        \cmidrule(lr){3-4} \cmidrule(lr){5-6} \cmidrule(lr){7-8}
         & & Test acc & $\bQ_{Y|D}$ err & Test acc & $\bQ_{Y|D}$ err & Test acc & $\bQ_{Y|D}$ err \\

\midrule
\multirow{2}{*}{10} &SCAN & 0.810 $\pm$ 0.006  & 0.145 $\pm$ 0.011  & \textbf{0.820} $\pm$ 0.006  & 0.129 $\pm$ 0.013  & \textbf{0.817} $\pm$ 0.006  & 0.085 $\pm$ 0.009  \\
&DDFA (RI) & 0.707 $\pm$ 0.052  & 0.042 $\pm$ 0.008  & 0.586 $\pm$ 0.055  & 0.046 $\pm$ 0.010  & 0.316 $\pm$ 0.036  & 0.069 $\pm$ 0.005  \\
&DDFA (SI) & \textbf{0.893} $\pm$ 0.043  & \textbf{0.021} $\pm$ 0.004  & 0.795 $\pm$ 0.048  & \textbf{0.037} $\pm$ 0.006  & 0.634 $\pm$ 0.027  & \textbf{0.051} $\pm$ 0.003  \\
\midrule
\multirow{2}{*}{15} &SCAN & 0.822 $\pm$ 0.015  & 0.152 $\pm$ 0.012  & 0.820 $\pm$ 0.013  & 0.123 $\pm$ 0.005  & \textbf{0.817} $\pm$ 0.007  & 0.085 $\pm$ 0.006  \\
&DDFA (RI) & 0.724 $\pm$ 0.072  & 0.044 $\pm$ 0.021  & 0.569 $\pm$ 0.042  & 0.049 $\pm$ 0.010  & 0.279 $\pm$ 0.085  & 0.084 $\pm$ 0.032  \\
&DDFA (SI) & \textbf{0.925} $\pm$ 0.061  & \textbf{0.018} $\pm$ 0.011  & \textbf{0.877} $\pm$ 0.049  & \textbf{0.021} $\pm$ 0.008  & 0.732 $\pm$ 0.064  & \textbf{0.035} $\pm$ 0.007  \\
\midrule
\multirow{2}{*}{20} &SCAN & 0.819 $\pm$ 0.015  & 0.149 $\pm$ 0.009  & 0.817 $\pm$ 0.009  & 0.119 $\pm$ 0.007  & \textbf{0.807} $\pm$ 0.009  & 0.081 $\pm$ 0.006  \\
&DDFA (RI) & 0.725 $\pm$ 0.055  & 0.042 $\pm$ 0.011  & 0.507 $\pm$ 0.036  & 0.054 $\pm$ 0.006  & 0.257 $\pm$ 0.052  & 0.072 $\pm$ 0.009  \\
&DDFA (SI) & \textbf{0.898} $\pm$ 0.072  & \textbf{0.026} $\pm$ 0.010  & \textbf{0.889} $\pm$ 0.034  & \textbf{0.022} $\pm$ 0.005  & 0.765 $\pm$ 0.042  & \textbf{0.031} $\pm$ 0.006  \\
\midrule
\multirow{2}{*}{25} &SCAN & 0.801 $\pm$ 0.016  & 0.147 $\pm$ 0.017  & 0.819 $\pm$ 0.008  & 0.114 $\pm$ 0.003  & \textbf{0.816} $\pm$ 0.013  & 0.083 $\pm$ 0.005  \\
&DDFA (RI) & 0.709 $\pm$ 0.105  & 0.051 $\pm$ 0.027  & 0.552 $\pm$ 0.039  & 0.047 $\pm$ 0.008  & 0.274 $\pm$ 0.061  & 0.070 $\pm$ 0.009  \\
&DDFA (SI) & \textbf{0.941} $\pm$ 0.048  & \textbf{0.019} $\pm$ 0.007  & \textbf{0.912} $\pm$ 0.013  & \textbf{0.017} $\pm$ 0.005  & 0.796 $\pm$ 0.057  & \textbf{0.029} $\pm$ 0.006  \\

        \bottomrule
    \end{tabular}
    }
\end{table}

\begin{table}[h]
    \caption{\emph{Full results on CIFAR-20}. Each entry is produced with the averaged result of 5 different random seeds, formatted as mean $\pm$ standard deviation. With DDFA (RI) we refer to DDFA with randomly initialized backbone.  With DDFA (SI) we refer to DDFA's backbone initialized with SCAN. Note that in DDFA (SI), we do not leverage SCAN for clustering. $\alpha$ is the Dirichlet parameter used for generating label marginals in each domain, $\kappa$ is the maximum allowed condition number of the generated $\bQ_{Y|D}$ matrix, $r$ is number of domains.}
    \vspace{5pt}
    \label{table:cifar20_appendix}
    \centering
    \scalebox{0.78}{
    \begin{tabular}{llllllll}
        \toprule
        \multirow{2}{*}{r} & \multirow{2}{*}{Approaches}    &\multicolumn{2}{c}{$\alpha: 0.5, \; \kappa: 8$} & \multicolumn{2}{c}{$\alpha: 3, \; \kappa: 12$} & \multicolumn{2}{c}{$\alpha: 10, \; \kappa: 20$} \\
        \cmidrule(lr){3-4} \cmidrule(lr){5-6} \cmidrule(lr){7-8}
         & & Test acc & $\bQ_{Y|D}$ err & Test acc & $\bQ_{Y|D}$ err & Test acc & $\bQ_{Y|D}$ err \\
\midrule
\multirow{2}{*}{20} &SCAN & 0.445 $\pm$ 0.025  & 0.090 $\pm$ 0.002  & 0.432 $\pm$ 0.018  & 0.081 $\pm$ 0.002  & \textbf{0.438} $\pm$ 0.010  & 0.062 $\pm$ 0.002  \\
&DDFA (RI) & 0.549 $\pm$ 0.081  & 0.038 $\pm$ 0.006  & 0.342 $\pm$ 0.013  & 0.043 $\pm$ 0.004  & 0.203 $\pm$ 0.019  & 0.053 $\pm$ 0.003  \\
&DDFA (SI) & \textbf{0.807} $\pm$ 0.034  & \textbf{0.021} $\pm$ 0.005  & \textbf{0.582} $\pm$ 0.060  & \textbf{0.028} $\pm$ 0.004  & 0.355 $\pm$ 0.108  & \textbf{0.035} $\pm$ 0.002  \\
\midrule
\multirow{2}{*}{25} &SCAN & 0.451 $\pm$ 0.040  & 0.091 $\pm$ 0.002  & 0.457 $\pm$ 0.008  & 0.079 $\pm$ 0.002  & 0.444 $\pm$ 0.018  & 0.061 $\pm$ 0.002  \\
&DDFA (RI) & 0.542 $\pm$ 0.037  & 0.040 $\pm$ 0.002  & 0.283 $\pm$ 0.028  & 0.050 $\pm$ 0.003  & 0.179 $\pm$ 0.008  & 0.053 $\pm$ 0.002  \\
&DDFA (SI) & \textbf{0.851} $\pm$ 0.041  & \textbf{0.017} $\pm$ 0.003  & \textbf{0.667} $\pm$ 0.035  & \textbf{0.024} $\pm$ 0.002  & \textbf{0.495} $\pm$ 0.039  & \textbf{0.031} $\pm$ 0.004  \\
\midrule
\multirow{2}{*}{30} &SCAN & 0.453 $\pm$ 0.010  & 0.088 $\pm$ 0.004  & 0.436 $\pm$ 0.018  & 0.079 $\pm$ 0.003  & 0.439 $\pm$ 0.014  & 0.061 $\pm$ 0.002  \\
&DDFA (RI) & 0.486 $\pm$ 0.078  & 0.046 $\pm$ 0.006  & 0.287 $\pm$ 0.046  & 0.054 $\pm$ 0.006  & 0.123 $\pm$ 0.045  & 0.068 $\pm$ 0.015  \\
&DDFA (SI) & \textbf{0.868} $\pm$ 0.044  & \textbf{0.016} $\pm$ 0.004  & \textbf{0.687} $\pm$ 0.043  & \textbf{0.024} $\pm$ 0.004  & \textbf{0.517} $\pm$ 0.022  & \textbf{0.032} $\pm$ 0.002  \\

        \bottomrule
    \end{tabular}
    }
\end{table}  

\begin{table}[ht]
    \caption{\emph{Results on ImageNet-50}. Each entry is produced with the averaged result of 5 different random seeds, formatted as mean $\pm$ standard deviation. With DDFA (SI) we refer to DDFA's backbone initialized with SCAN. Note that in DDFA (SI), we do not leverage SCAN for clustering. $\alpha$ is the Dirichlet parameter used for generating label marginals in each domain, $\kappa$ is the maximum allowed condition number of the generated $\bQ_{Y|D}$ matrix, $r$ is number of domains. \update{``Test acc'' is classification accuracy, under the best permutation of the recovered classes, and ``$\bQ_{Y|D}$ err'' is the average entry-wise absolute error in the recovered $\bQ_{Y|D}$.}}
    \centering
    \scalebox{0.78}{
    \begin{tabular}{llllllll}
        \toprule
        \multirow{2}{*}{r} & \multirow{2}{*}{Approaches}    &\multicolumn{2}{c}{$\alpha: 0.5, \; \kappa: 200$} & \multicolumn{2}{c}{$\alpha: 3, \; \kappa: 205$} & \multicolumn{2}{c}{$\alpha: 10, \; \kappa: 210$} \\
        \cmidrule(lr){3-4} \cmidrule(lr){5-6} \cmidrule(lr){7-8}
         & & Test acc & $\bQ_{Y|D}$ err & Test acc & $\bQ_{Y|D}$ err & Test acc & $\bQ_{Y|D}$ err \\

\midrule
\multirow{2}{*}{50} &SCAN & \textbf{0.734} $\pm$ 0.052  & 0.039 $\pm$ 0.000  & \textbf{0.751} $\pm$ 0.028  & 0.037 $\pm$ 0.001  & \textbf{0.736} $\pm$ 0.015  & 0.032 $\pm$ 0.001  \\
&DDFA (SI) & 0.722 $\pm$ 0.058  & \textbf{0.011} $\pm$ 0.003  & 0.594 $\pm$ 0.062  & \textbf{0.015} $\pm$ 0.002  & 0.452 $\pm$ 0.059  & \textbf{0.019} $\pm$ 0.000  \\
\midrule
\multirow{2}{*}{60} &SCAN & 0.747 $\pm$ 0.024  & 0.039 $\pm$ 0.000  & \textbf{0.756} $\pm$ 0.027  & 0.037 $\pm$ 0.000  & \textbf{0.737} $\pm$ 0.011  & 0.032 $\pm$ 0.000  \\
&DDFA (SI) & \textbf{0.800} $\pm$ 0.066  & \textbf{0.010} $\pm$ 0.002  & 0.696 $\pm$ 0.033  & \textbf{0.014} $\pm$ 0.002  & 0.586 $\pm$ 0.054  & \textbf{0.017} $\pm$ 0.002  \\

\bottomrule

    \end{tabular}
    }
\end{table}

\begin{table}[ht]
    \caption{\emph{Full results on FieldGuide-2}. \update{Each entry is produced with the averaged result of 5 different random seeds, formatted as mean $\pm$ standard deviation. With DDFA (RI) we refer to DDFA with randomly initialized backbone. With DDFA (SPI) we refer to DDFA initialized with pretext training adopted by SCAN. Note that in DDFA (SPI), we do not leverage SCAN for clustering. $\alpha$ is the Dirichlet parameter used for generating label marginals in each domain, $\kappa$ is the maximum allowed condition number of the generated $\bQ_{Y|D}$ matrix, $r$ is number of domains.}} \label{table:fieldguide2}
    \centering
    \scalebox{0.78}{
    \begin{tabular}{llllllll}
        \toprule
        \multirow{2}{*}{r} & \multirow{2}{*}{Approaches}    &\multicolumn{2}{c}{$\alpha: 0.5, \; \kappa: 3$} & \multicolumn{2}{c}{$\alpha: 3, \; \kappa: 5$} & \multicolumn{2}{c}{$\alpha: 10, \; \kappa: 7$} \\
        \cmidrule(lr){3-4} \cmidrule(lr){5-6} \cmidrule(lr){7-8}
         & & Test acc & $\bQ_{Y|D}$ err & Test acc & $\bQ_{Y|D}$ err & Test acc & $\bQ_{Y|D}$ err \\

\midrule
\multirow{2}{*}{2} &SCAN & 0.589 $\pm$ 0.006  & 0.880 $\pm$ 0.117  & 0.591 $\pm$ 0.021  & 0.372 $\pm$ 0.160  & 0.601 $\pm$ 0.016  & 0.283 $\pm$ 0.068  \\
&DDFA (SPI) & \textbf{0.947} $\pm$ 0.036  & \textbf{0.084} $\pm$ 0.043  & \textbf{0.725} $\pm$ 0.133  & \textbf{0.150} $\pm$ 0.098  & \textbf{0.676} $\pm$ 0.043  & \textbf{0.222} $\pm$ 0.046  \\
\midrule
\multirow{2}{*}{3} &SCAN & 0.608 $\pm$ 0.022  & 0.706 $\pm$ 0.069  & 0.591 $\pm$ 0.025  & 0.426 $\pm$ 0.155  & 0.585 $\pm$ 0.026  & 0.282 $\pm$ 0.135  \\
&DDFA (SPI) & \textbf{0.922} $\pm$ 0.033  & \textbf{0.086} $\pm$ 0.053  & \textbf{0.818} $\pm$ 0.050  & \textbf{0.152} $\pm$ 0.073  & \textbf{0.619} $\pm$ 0.116  & \textbf{0.264} $\pm$ 0.122  \\
\midrule
\multirow{2}{*}{5} &SCAN & 0.589 $\pm$ 0.026  & 0.640 $\pm$ 0.108  & 0.587 $\pm$ 0.013  & 0.389 $\pm$ 0.190  & 0.586 $\pm$ 0.008  & \textbf{0.226} $\pm$ 0.081  \\
&DDFA (SPI) & \textbf{0.857} $\pm$ 0.069  & \textbf{0.181} $\pm$ 0.116  & \textbf{0.738} $\pm$ 0.137  & \textbf{0.194} $\pm$ 0.128  & \textbf{0.610} $\pm$ 0.102  & 0.253 $\pm$ 0.153  \\
\midrule
\multirow{2}{*}{7} &SCAN & 0.571 $\pm$ 0.023  & 0.679 $\pm$ 0.111  & 0.586 $\pm$ 0.022  & 0.422 $\pm$ 0.058  & 0.581 $\pm$ 0.005  & 0.244 $\pm$ 0.070  \\
&DDFA (SPI) & \textbf{0.848} $\pm$ 0.051  & \textbf{0.223} $\pm$ 0.063  & \textbf{0.779} $\pm$ 0.057  & \textbf{0.166} $\pm$ 0.098  & \textbf{0.645} $\pm$ 0.128  & \textbf{0.218} $\pm$ 0.113  \\
\midrule
\multirow{2}{*}{10} &SCAN & 0.587 $\pm$ 0.018  & 0.716 $\pm$ 0.138  & 0.589 $\pm$ 0.003  & 0.347 $\pm$ 0.043  & 0.590 $\pm$ 0.003  & 0.193 $\pm$ 0.031  \\
&DDFA (SPI) & \textbf{0.886} $\pm$ 0.023  & \textbf{0.162} $\pm$ 0.082  & \textbf{0.744} $\pm$ 0.060  & \textbf{0.194} $\pm$ 0.062  & \textbf{0.606} $\pm$ 0.063  & \textbf{0.172} $\pm$ 0.087  \\

        \bottomrule
    \end{tabular}
    }
\end{table}

\begin{table}[ht]
    \caption{\emph{Results on FieldGuide-28}. Each entry is produced with the averaged result of 5 different random seeds, formatted as mean $\pm$ standard deviation. With DDFA (SPI) we refer to DDFA initialized with pretext training adopted by SCAN. Note that in DDFA (SPI), we do not leverage SCAN for clustering. $\alpha$ is the Dirichlet parameter used for generating label marginals in each domain, $\kappa$ is the maximum allowed condition number of the generated $\bQ_{Y|D}$ matrix, $r$ is number of domains. \update{``Test acc'' is classification accuracy, under the best permutation of the recovered classes, and ``$\bQ_{Y|D}$ err'' is the average entry-wise absolute error in the recovered $\bQ_{Y|D}$.}}
    \label{table:fieldguide28}
    \centering
    \scalebox{0.78}{
    \begin{tabular}{llllllll}
        \toprule
        \multirow{2}{*}{r} & \multirow{2}{*}{Approaches}    &\multicolumn{2}{c}{$\alpha: 0.5, \; \kappa: 12$} & \multicolumn{2}{c}{$\alpha: 3, \; \kappa: 20$} & \multicolumn{2}{c}{$\alpha: 10, \; \kappa: 28$} \\
        \cmidrule(lr){3-4} \cmidrule(lr){5-6} \cmidrule(lr){7-8}
         & & Test acc & $\bQ_{Y|D}$ err & Test acc & $\bQ_{Y|D}$ err & Test acc & $\bQ_{Y|D}$ err \\
\midrule
\multirow{2}{*}{28} &SCAN & 0.257 $\pm$ 0.004  & 0.066 $\pm$ 0.002  & 0.250 $\pm$ 0.010  & 0.061 $\pm$ 0.002  & 0.249 $\pm$ 0.007  & 0.049 $\pm$ 0.002  \\
&DDFA (SPI) & \textbf{0.577} $\pm$ 0.078  & \textbf{0.029} $\pm$ 0.005  & \textbf{0.395} $\pm$ 0.090  & \textbf{0.032} $\pm$ 0.005  & \textbf{0.262} $\pm$ 0.079  & \textbf{0.036} $\pm$ 0.002  \\
\midrule
\multirow{2}{*}{37} &SCAN & 0.255 $\pm$ 0.026  & 0.069 $\pm$ 0.001  & 0.265 $\pm$ 0.016  & 0.061 $\pm$ 0.001  & 0.251 $\pm$ 0.011  & 0.050 $\pm$ 0.001  \\
&DDFA (SPI) & \textbf{0.742} $\pm$ 0.017  & \textbf{0.031} $\pm$ 0.003  & \textbf{0.530} $\pm$ 0.037  & \textbf{0.032} $\pm$ 0.003  & \textbf{0.336} $\pm$ 0.038  & \textbf{0.036} $\pm$ 0.002  \\
\midrule
\multirow{2}{*}{42} &SCAN & 0.271 $\pm$ 0.021  & 0.069 $\pm$ 0.001  & 0.261 $\pm$ 0.011  & 0.061 $\pm$ 0.001  & 0.257 $\pm$ 0.013  & 0.049 $\pm$ 0.002  \\
&DDFA (SPI) & \textbf{0.703} $\pm$ 0.041  & \textbf{0.032} $\pm$ 0.003  & \textbf{0.475} $\pm$ 0.044  & \textbf{0.037} $\pm$ 0.002  & \textbf{0.371} $\pm$ 0.027  & \textbf{0.033} $\pm$ 0.003  \\
\midrule
\multirow{2}{*}{47} &SCAN & 0.248 $\pm$ 0.013  & 0.069 $\pm$ 0.002  & 0.267 $\pm$ 0.018  & 0.061 $\pm$ 0.002  & 0.249 $\pm$ 0.012  & 0.049 $\pm$ 0.001  \\
&DDFA (SPI) & \textbf{0.695} $\pm$ 0.026  & \textbf{0.033} $\pm$ 0.002  & \textbf{0.487} $\pm$ 0.027  & \textbf{0.036} $\pm$ 0.003  & \textbf{0.374} $\pm$ 0.013  & \textbf{0.035} $\pm$ 0.001  \\

        \bottomrule
    \end{tabular}
    }
\end{table}

%% file: sections/appendix_F_geometry.tex
\begin{figure}[t]

  \centering
  \includegraphics[width=\textwidth]{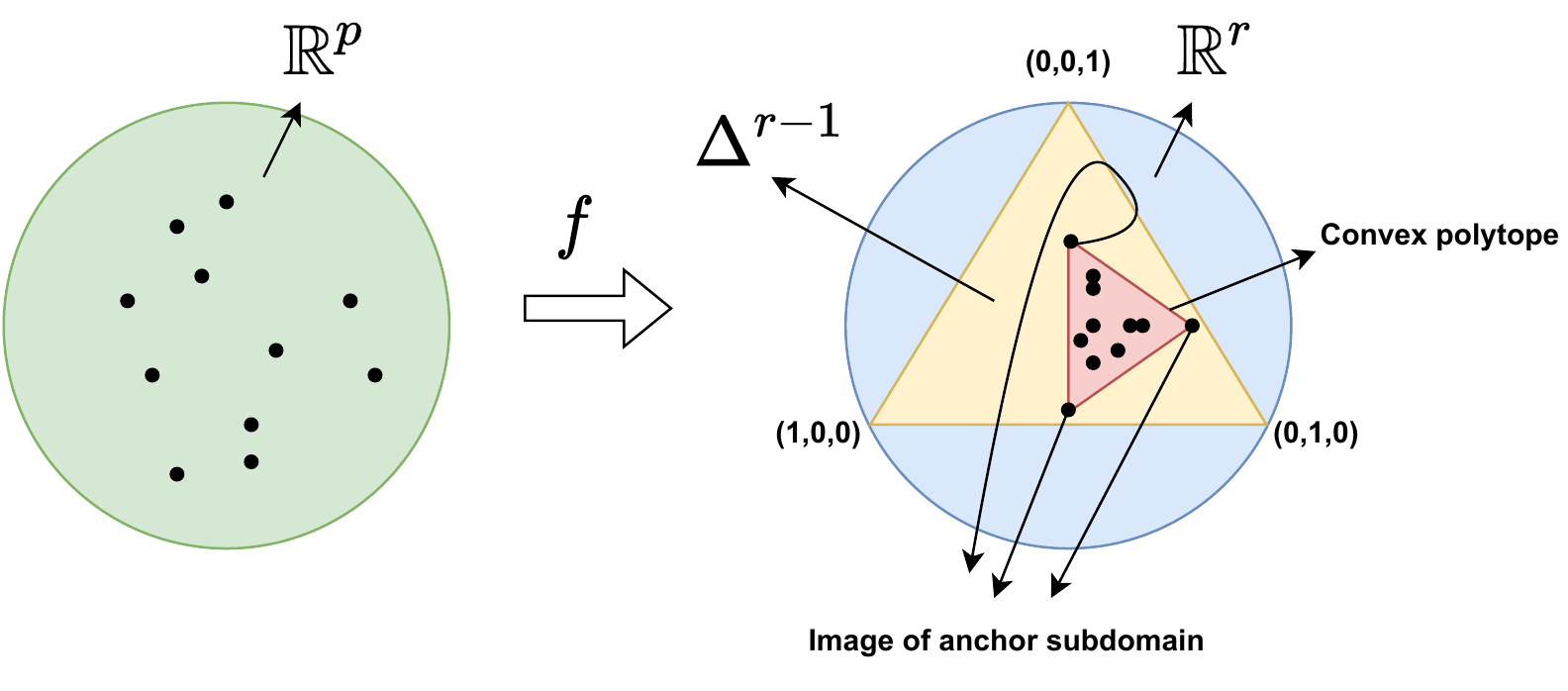}
  \caption{This figure illustrates the case with 3 domains and 3 classes. The oracle domain discriminator maps points from a high-dimensional input space to a $k = 3$ vertex convex polytope (shaded red) embedded in $\Delta^{r-1} , r = 3$ (shaded yellow). The anchor subdomains map to the vertices of this polytope.}
  \label{fig:polytope}
\end{figure}

The geometric properties of topic modeling for finite, discrete random variables has been explored in depth in related works (\citep{huang2016anchor, donoho, chen2021learning}). The observation that columns in $\bQ_{X|D}$ are convex combinations of columns in $\bQ_{X|Y}$ leads to a perspective on identification of the matrix decomposition as identification of the convex polytope in $\mathbb{R}^m$ which encloses all of the columns of $\bQ_{X|D}$ (the corners of which correspond to columns of $\bQ_{X|Y}$ under certain identifiability conditions).

Here, we briefly discuss an interesting but somewhat different application of convex polytope geometry. Instead of a convex polytope in $\mathbb{R}^m$ with corners as columns of $\bQ_{X|Y}$, we concern ourselves with the convex polytope in $\mathbb{R}^r$ with corners as columns in $\bQ_{D|Y}$, which must enclose all values taken by the oracle domain discriminator $f(x)$ for $x \in \mathcal{X}, q(x) > 0$.

\smallskip

Let us assume that Assumptions A.1--A.4 are satisfied.
We recall the oracle domain discriminator $f$ which is defined such that $[f(x)]_d = q(d|X = x)$. Let $x \in \mathcal{X} = \mathbb{R}^p$. Now, since the $r$ values $q(d|X = x)$ for $ d \in \{1,2,...,r \} $ together constitute a categorical distribution, each of these $r$ values lie between $0$ and $1$, and also their sum adds to 1. Therefore the vector $f(x)$ lies on the simplex $\Delta^{r-1}$. We now express $f(x)$ as a convex combination of the $k$ columns of $\bQ_{D|Y}$. We denote these column vectors $\bQ_{D|Y}[\boldsymbol :, y]$ for each $y \in \out = [k]$. Note that each such vector also lies in the $\Delta^{r-1}$ simplex.

As an intermediate step in the proof of Lemma \ref{lemma:dd_maps_to_unique} given in \appref{sec:lemma-proofs}, we showed that each $f(x)$ is a linear combination of these columns of $\bQ_{D|Y}$ with coefficients $q(y | X = x)$ for all $y \in \out$.
That is, we can rewrite $f(x) = \bQ_{D|Y} \left[ Q(Y = 1 | X = x) \; ... \; Q(Y = k | X = x)\right]^\top$

Since the coefficients in the linear combination are probabilities which, taken together, form a categorical distribution, they lie between $0$ and $1$ and sum to 1. Thus, for all $x \in \calX$ with $q(x) > 0$, $f(x)$ can be expressed as a \emph{convex} combination of the columns of $\bQ_{D|Y}$. Therefore, for any $x$ with $q(x) > 0$, $f(x)$ lies inside the $k-$vertex convex polytope with corners as the columns of $\bQ_{D|Y}$ (which are linearly independent by Lemma \ref{lemma:P_d_y-linearly-independent}). This polytope is embedded in $\Delta^{r-1}$.

Now consider $x$ in an anchor sub-domain, that is $x \in A_y$ for some $y \in \out$. We know that if $q(x) > 0$, $q(y|X = x) = 1$, $q(y'|X = x) = 0$ for all $y' \neq y$ (Lemma \ref{lemma:onehot}).
Since the $q(y|X = x)$ are now one-hot, we have that $f(x) = \bQ_{D|Y}[\boldsymbol :, y]$ for $x \in A_y$. In words, this means that $f(A_y)$ is precisely the $y$th column of $\bQ_{D|Y}$. It follows that the domain discriminator maps each of the $k$ anchor sub-domains exactly to a unique vertex of the polytope. The situation is described in \figref{fig:polytope}.

We could now recover the columns of $\bQ_{D|Y}$, up to permutation, with the following procedure: 
\begin{enumerate}
    \item Push all $x \in \calX$ through $f$. 
    \item Find the minimum volume convex polytope that contains the resulting density of points on the simplex. The vectors that compose the vertices of this polytope are the columns of $\bQ_{D|Y}$, up to permutation.
\end{enumerate}

Note that from Assumption A.4, we are guaranteed to have a region of the input space with at least $\epsilon > 0$ mass that gets mapped to each of the vertices when carrying out step (i). Therefore, our discovered minimum volume polytope must enclose all of these vertices. Since no mass will exist outside of the true polytope, requiring a minimum volume polytope will ensure that the recovered polytope fits the true polytope's vertices precisely (as any extraneous volume outside of the true polytope must be eliminated). Then step (ii) recovers $\bQ_{D|Y}$, up to permutation of columns.
Having recovered $\bQ_{D|Y}$, we can use Lemmas \ref{lemma:solve-domain-agnostic} and \ref{lemma:solve-domain-informed} to recover $q(y|x,d)$.

This procedure is a geometric alternative to the clustering approach outlined in Algorithm $\ref{alg:1}$. In practice, fitting a convex hull around the outputs of a noisy, non-oracle estimated domain discriminator may be computationally expensive, and noise may lead this sensitive procedure \update{to} fail to recover the true vertices.

%% file: sections/appendix_G_ablation.tex
\update{We conduct an ablation on the choice of $m$, the parameter indicating how many clusters to find in the $q(d|x)$ space. We use the CIFAR-20 dataset with 20 domains and employ DDFA (SI) and SCAN models, following the same hyperparameters as outlined in  \appref{sec:appendix-experimental-details}, except for modifying the choice of $m$ for DDFA (SI). Results are obtained as the average of three random seeds.}

\begin{figure}[ht]
  \centering
  \includegraphics[width=0.95\textwidth]{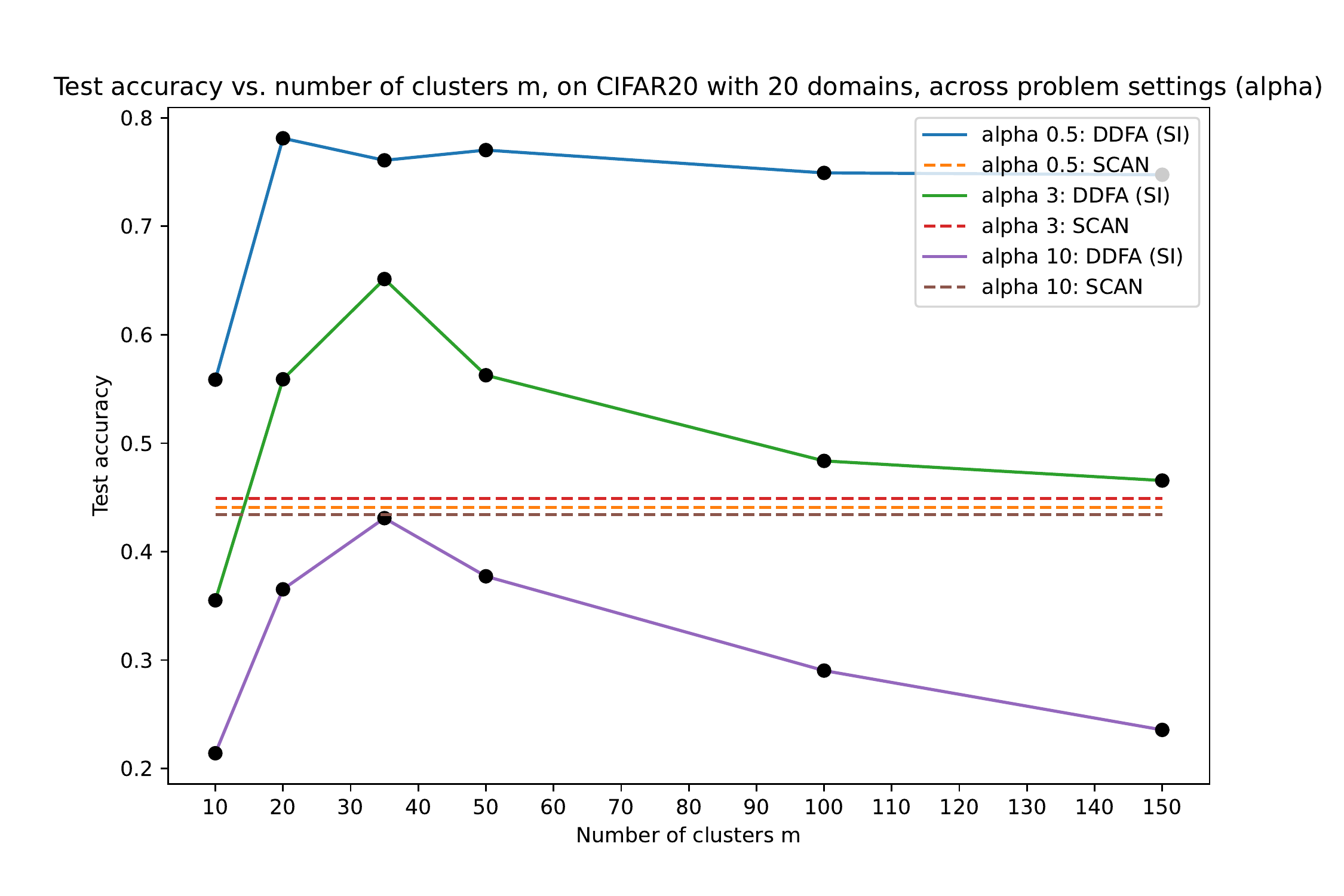}
  \caption{\update{Test accuracy of DDFA (SI) approach and SCAN baseline on CIFAR20 with 20 domains, while modifying the number of clusters $m$ for DDFA (SI). The choice of $\alpha$ roughly modifies the difficulty of the problem, where small $\alpha$ is easier. We note that typically we require choice of $m \ge k$. We portray one datapoint where this constraint is violated, and $m = 10$. Black dots indicate tested values of $m$, and lines are plotted only to show the trend. Larger accuracy is better.}}
  \label{fig:ablation_test_acc}
\end{figure}

\begin{figure}[ht]
  \centering
  \includegraphics[width=0.95\textwidth]{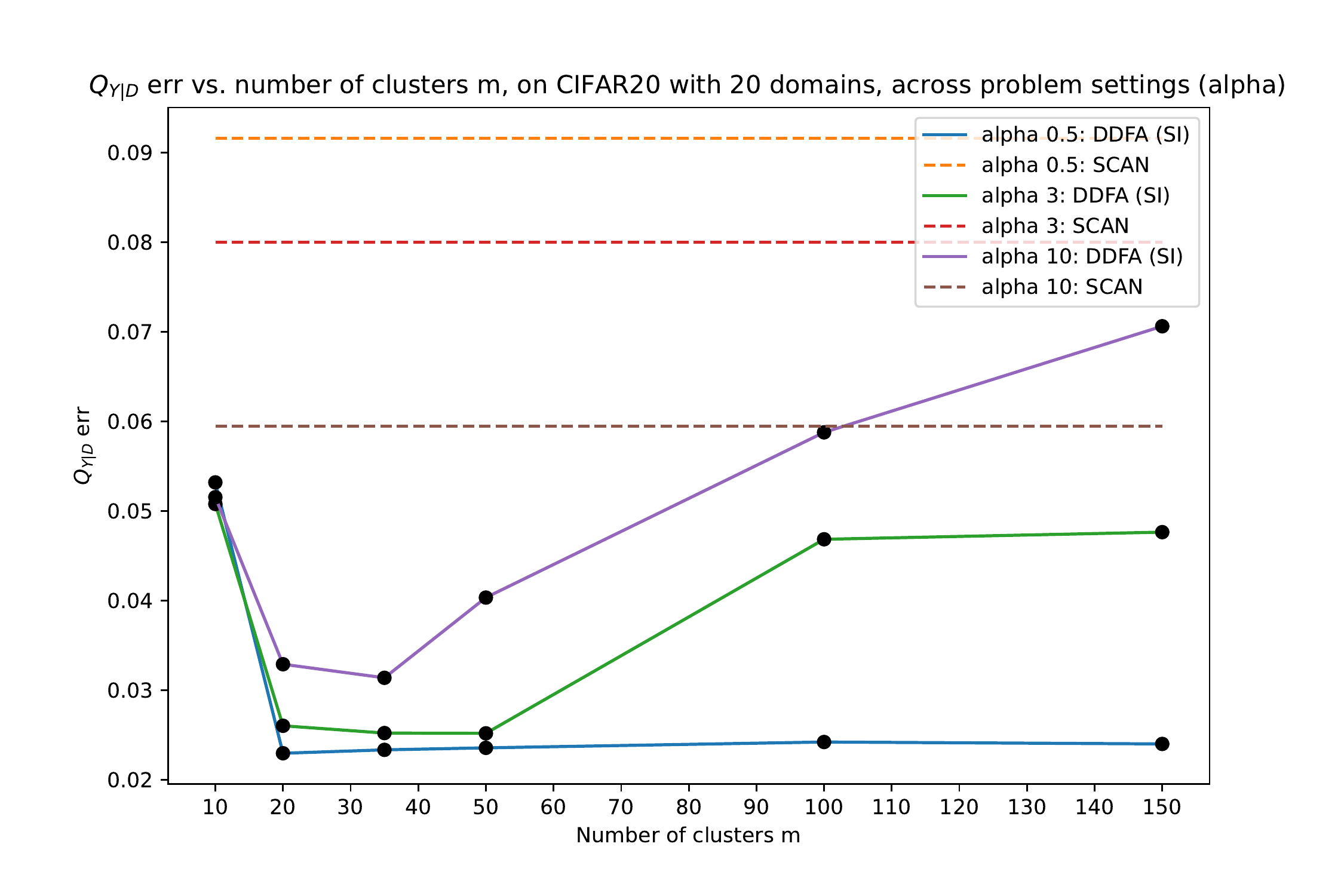}
  \caption{\update{Element-wise average absolute $Q_{Y|D}$ reconstruction error of DDFA (SI) approach and SCAN baseline on CIFAR-20 with 20 domains, while modifying the number of clusters $m$ for DDFA (SI). The choice of $\alpha$ roughly modifies the difficulty of the problem, where small $\alpha$ is easier. We note that typically we require choice of $m \ge k$. We portray one datapoint where this constraint is violated, and $m = 10$. Black dots indicate tested values of $m$, and lines are plotted only to show the trend. Smaller error is better.}}
  \label{fig:ablation_q_y_d_err}
\end{figure}

\update{The number of true classes is 20 in CIFAR-20. As seen in \figref{fig:ablation_test_acc}, when $m$ is chosen to be 10, violating the typical constraint that $m \ge k$, we can still solve for the solution, but we get poor performance, seeing a drop in accuracy as much as 20\% from a better-chosen value of $m$. Choosing $m$ directly equal to or slightly larger than $k$ provide the best performance, with a slope-off in performance at very large $m$.}

\update{The trend is roughly mirrored in \figref{fig:ablation_q_y_d_err}, which shows how the reconstruction error changes over the same variation in $m$. Under all settings, using $m = 10 < k$ clusters provides a poor reconstruction, while the best reconstruction is found with $m$ roughly equal to $k$ or slightly larger. Performance degrades as $m$ grows very large, although the effect is very slight for the alpha = 0.5 setting.}

\update{Intuitively, these results show that breaking the $m \ge k$ condition not only violates the theoretical identifiability, but also leads to poorer empirical performance; choosing very large $m$ can also lead to degraded performance, likely due to propagation of inaccuracies in the finite-sample estimation of the $\hat{Q}_{c(X)|D}$ matrix}.

%% file: sections/appendix_H_ablation_naive.tex
\update{One might ask whether the semantic meaning of the domain discrimination space is necessary in DDFA; might we exchange the domain discrimination step in Algorithm \ref{alg:1} for a naive step in which we simply pass the input through an arbitrary feature extractor and then proceed to clustering in this space?}

\update{The first remark we make is that the domain discriminator does not purely provide a clustering representation; its semantic meaning is also important for the computation of the final domain-adjusted $p_d(y|x)$ prediction, as a reliable estimate of $q(d|x)$, in conjunction with the estimate of the $Q_{Y|D}$ matrix, allow us to estimate $p_d(y|x)$ via Algorithm \ref{alg:2}. Without this semantic meaning, our class prediction can be based only on a coarse prediction at the level of the \textit{cluster}, not the individual datapoint.}

\update{However, if we are still determined to use an alternate representation, it is indeed possible to do so. We illustrate a variant of DDFA using a naïve representation in Algorithms \ref{alg:1-naive} and \ref{alg:2-naive}, and then evaluate this procedure on CIFAR-20 as an ablative study on DDFA. We compare naïve results with standard DDFA results, and with a traditional SCAN baseline, in \tabref{table:cifar20_ablation_naive}.}

\setlength{\textfloatsep}{12pt}
\begin{algorithm}[t]
\caption{DDFA (Naïve) Training}
\label{alg:1-naive}
\begin{algorithmic}[1]

\INPUT $k \geq 1, r \geq k, \{(x_i,d_i)\}_{i \in [n]} \sim q(x, d), \textrm{ A naive representation function } \mathcal{\phi} \textrm{ from } \mathbb{R}^p \to \mathbb{R}^r $

\STATE $\textrm{Push all } \{x_i\}_{i \in [n]} \textrm{ through } \phi $.

\STATE $\textrm{Train clustering algorithm on the n points } \{\phi(x_i)\}_{i \in [n]}, \textrm{ obtain }m \textrm{ clusters.}$

\STATE $c(x_i) \gets \textrm{Cluster id of } \phi(x_i)$

\STATE $\hat{q}(c(X) = a | D = b) \gets  \frac{\sum_{i \in [n]} \mathbb{I}[c(x_i) = a, \; d_i = b ]}{\sum_{j \in [n]}\mathbb{I}[ d_j = b ]}$

\STATE $ \textrm{Populate } \hat{\bQ}_{c(X)|D} \textrm{ as } [\hat{\bQ}_{c(X)|D}]_{a, b} \gets \hat{q}(c(X) = a | D = b)$

\STATE $\hat{\bQ}_{c(X)|Y} , \hat{\bQ}_{Y|D} \gets \textrm{ NMF }(\hat{\bQ}_{c(X)|D} )$

\OUTPUT {$\hat{\bQ}_{c(X)|Y}, \hat{\bQ}_{Y|D}, \textrm{ clustering discretization function } c$}

\end{algorithmic}

\end{algorithm}

\update{\textbf{Naïve Representation Variant of DDFA {} {} } The only major changes from the original DDFA for Algorithm \ref{alg:1-naive} are the removal of the need to train any $\hat{f}$ domain discriminator, the use of the arbitrary representation space $\phi$ before clustering, and the reliance on the output of the clustering discretization function $c$ as well as $\hat{\bQ}_{c(X)|Y}$ (which are both discarded in the original procedure). We need $c$ and $\hat{\bQ}_{c(X)|Y}$ because in Algorithm \ref{alg:2-naive} we will use them for domain-adjusted class prediction.}

\begin{algorithm}[ht]
\caption{DDFA (Naïve) Prediction}
\label{alg:2-naive} 
\begin{algorithmic}[1]

\INPUT $\hat{\bQ}_{c(X)|Y}, \hat{\bQ}_{Y|D}, \textrm{ clustering discretization function } c, (x', d') \sim q(x,d)$ 

\STATE $\textrm{Pass } x' \textrm{ through } c \textrm{ to get cluster id } c(x')$.

\STATE $\hat{q}(c(x') | Y = y'') \gets [\hat{\bQ}_{c(X)|Y}]_{c(x'),y''}  \textrm{ for all } y''$

\STATE $\hat{q}(y''|D = d') \gets [\hat{\bQ}_{Y|D}]_{y'',d'} \textrm{ for all } y'' $

\STATE $\hat{q}(y|c(X) = c(x'), D = d') \gets \cfrac{\hat{q}(c(x') | Y = y)\hat{q}(y|D = d')  }{  \sum\limits_{y'' \in \out} \hat{q}(c(x') | Y = y'')\hat{q}(y''|D=d')}$

\STATE $y_\textrm{pred} \gets \argmax_{y \in [k]} \hat{q}(y|c(X) = c(x'), D = d')$

\OUTPUT : { $\hat{q}(y|c(X) = c(x'), D = d') = \hat{p}_{d'}(y|c(x'))$, $y_\textrm{pred}$ } 
\end{algorithmic}

\end{algorithm}

\update{Algorithm \ref{alg:2-naive} includes significant changes from the DDFA procedure. Since we do not have the estimate of $q(d|x)$ to use, we cannot directly reason about how different locations in the representation space induced by $\phi$ correspond to different probabilities of class labels. However, because we have $\hat{\bQ}_{c(X)|Y}$ and $\hat{\bQ}_{Y|D}$, two outputs of the NMF decomposition in Algorithm \ref{alg:1-naive}, we can calculate a coarse prediction over labels $y$ for each cluster, and then assign the same prediction to each point in that cluster. To obtain the closed-form for this coarse prediction $\hat{p}_{d} (y | c(x))$ used in Algorithm \ref{alg:2-naive}, we use the following derivation: }

\begin{align*}
    \hat{p}_d(y | c(x)) = \hat{q}(y | d, c(x)) &= \cfrac{\hat{q}(c(x) | y, d)\hat{q}(y, d)}{\hat{q}(d, c(x))} \\
&= \cfrac{\hat{q}(c(x) | y, d)\hat{q}(y, d)}{  \sum\limits_{y'' \in \out} \hat{q}(c(x) | y'', d)\hat{q}(y'', d) }
\end{align*}
By label shift, $q(c(x)|y,d) = q(c(x)|y)$, then
\begin{align*}
\hat{p}_d(y | c(x)) = \hat{q}(y | d, c(x)) &= \cfrac{\hat{q}(c(x) | y)\hat{q}(y,d)}{  \sum\limits_{y'' \in \out} \hat{q}(c(x) | y'')\hat{q}(y'',d)} \\
&= \cfrac{\hat{q}(c(x) | y)\hat{q}(y|d)\hat{q}(d)}{  \sum\limits_{y'' \in \out} \hat{q}(c(x) | y'')\hat{q}(y''|d)\hat{q}(d)} \\
&= \cfrac{\hat{q}(c(x) | y)\hat{q}(y|d)(\nicefrac{1}{r})}{  \sum\limits_{y'' \in \out} \hat{q}(c(x) | y'')\hat{q}(y''|d)(\nicefrac{1}{r})} \\
&= \cfrac{\hat{q}(c(x) | y)\hat{q}(y|d)  }{  \sum\limits_{y'' \in \out} \hat{q}(c(x) | y'')\hat{q}(y''|d)}
\end{align*}

\update{Combining Algorithms \ref{alg:1-naive} and \ref{alg:2-naive} allows us to empirically evaluate the behavior of an ablation on DDFA which does not use any domain discriminator. For a reasonable comparison, we need a meaningful naïve representation space. We use a SCAN pretrain backbone for ResNet-18, and remove the last linear layer in the ResNet-18 backbone in order to expose a 512-dimension representation space. Since clustering in high-dimensional spaces often performs poorly, we also map this 512-dimension representation down to only $r$ (the number of domains) dimensions using two different common dimensionality reduction methods: Independent Component Analysis (ICA) \cite{hyvarinen2000independent} and Principal Component Analysis (PCA) \cite{wold1987principal}. These smaller-dimension clustering problems provide a closer comparison to the dimensionality of the clustering problem in the DDFA (SI) procedure, for which we employ $m$ clusters. Note: we use ICA and PCA implementations from scikit-learn \cite{scikit-learn}.}

\update{SCAN, DDFA (RI), and DDFA (SI) experiment details are the same as explained in \appref{sec:appendix-experimental-details}; however, these are a different set of trials than those given in \secref{sec:exp}.}

\update{In general, we can see that all variants of the Naïve/ablated DDFA procedure perform worse than DDFA (SI) in all problem settings, over both metrics of interest. All variants of the Naïve procedure are also worse than DDFA (RI) in both metrics of interest when $\alpha$ is 0.5 or 3, although they match or outpace it slightly in some of the hardest settings ($\alpha = 10$). It is worth pointing out that this is predominantly due to the fact that DDFA (RI) performs notably poorly at harder settings.}

\update{The SCAN baseline outpaces these Naïve approaches, in terms of classification accuracy, in all problem settings. However, the Naïve approaches always achieve better $Q_{Y|D}$ reconstruction error.}

\begin{table}[ht]
    \caption{\update{\emph{Extended Results on CIFAR-20}. Each entry is produced with the averaged result of 3 different random seeds. With DDFA (RI) we refer to DDFA with randomly initialized backbone.  With DDFA (SI) we refer to DDFA's backbone initialized with SCAN. Note that in DDFA (SI), we do not leverage SCAN for clustering. With Naïve we refer to an ablation in which DDFA's domain discriminator is replaced with the SCAN pretrained backbone, with its final linear layer removed so that its output is a 512-dimension unsupervised representation space. With Naïve (ICA) and Naïve (PCA) we refer to similar ablations in which the activations from the second-to-last layer of SCAN network are mapped to $r$-dimensional space with ICA and PCA respectively. $\alpha$ is the Dirichlet parameter used for generating label marginals in each domain, $\kappa$ is the maximum allowed condition number of the generated $\bQ_{Y|D}$ matrix, $r$ is number of domains.}}
    \label{table:cifar20_ablation_naive}
    \begin{tabular}{llllllll}
        \toprule
        \multirow{2}{*}{r} & \multirow{2}{*}{Approaches}    &\multicolumn{2}{c}{$\alpha: 0.5, \; \kappa: 8$} & \multicolumn{2}{c}{$\alpha: 3, \; \kappa: 12$} & \multicolumn{2}{c}{$\alpha: 10, \; \kappa: 20$} \\
        \cmidrule(lr){3-4} \cmidrule(lr){5-6} \cmidrule(lr){7-8}
         & & Test acc & $\bQ_{Y|D}$ err & Test acc & $\bQ_{Y|D}$ err & Test acc & $\bQ_{Y|D}$ err \\
        \midrule
\multirow{3}{*}{20}	&	SCAN	    &	0.439	&	0.092	&	0.446	&	0.079	&	\textbf{0.434}	&	0.060	\\
	                &	DDFA (RI)	&	0.517	&	0.042	&	0.336	&	0.045	&	0.163	&	0.057	\\
                    &	DDFA (SI)	&	\textbf{0.784}	&	\textbf{0.023}	&	\textbf{0.593}	&	\textbf{0.027}	&	0.390	&	\textbf{0.034}	\\
                    
                    	&	Naïve 	&	0.231	&	0.075	&	0.156	&	0.065	&	0.116	&	0.054	\\
	&	Naïve (ICA)	&	0.225	&	0.073	&	0.137	&	0.070	&	0.105	&	0.056	\\
	&	Naïve (PCA)	&	0.204	&	0.076	&	0.141	&	0.065	&	0.108	&	0.055	\\

\midrule															
\multirow{3}{*}{25}	&	SCAN	    &	0.438	&	0.090	&	0.441	&	0.078	&	0.438	&	0.060	\\
                    &	DDFA (RI)	&	0.489	&	0.049	&	0.292	&	0.049	&	0.075	&	0.081	\\
	                &	DDFA (SI)	&	\textbf{0.837}	&	\textbf{0.020}	&	\textbf{0.669}	&	\textbf{0.025}	&	\textbf{0.487}	&	\textbf{0.030}	\\
	                
	&	Naïve	&	0.224	&	0.078	&	0.165	&	0.064	&	0.112	&	0.055	\\
	&	Naïve (ICA)	&	0.204	&	0.076	&	0.146	&	0.063	&	0.100	&	0.057	\\
	&	Naïve (PCA)	&	0.207	&	0.078	&	0.135	&	0.067	&	0.105	&	0.054	\\

\midrule															
\multirow{3}{*}{30}	&	SCAN	    &	0.432	&	0.094	&	0.457	&	0.077	&	0.431	&	0.059	\\
	                &	DDFA (RI)	&	0.512	&	0.046	&	0.299	&	0.048	&	0.087	&	0.077	\\
	                &	DDFA (SI)	&	\textbf{0.820}	&	\textbf{0.022}	&	\textbf{0.743}	&	\textbf{0.021}	&	\textbf{0.543}	&	\textbf{0.028}	\\

	&	Naïve	&	0.208	&	0.077	&	0.152	&	0.066	&	0.122	&	0.051	\\
	&	Naïve (ICA)	&	0.197	&	0.076	&	0.134	&	0.064	&	0.097	&	0.056	\\
	&	Naïve (PCA)	&	0.200	&	0.078	&	0.148	&	0.065	&	0.114	&	0.051	\\

        \bottomrule
    \end{tabular}
\end{table}